\documentclass{article}

\usepackage{microtype}
\usepackage{graphicx}
\usepackage{subfigure}
\usepackage{booktabs}
\usepackage{hyperref}

\usepackage[accepted]{icml2025}

\usepackage{amsmath}
\usepackage{amssymb}
\usepackage{mathtools}
\usepackage{amsthm}
\usepackage{thmtools}

\usepackage[capitalize,noabbrev]{cleveref}

\theoremstyle{plain}
\newtheorem{theorem}{Theorem}[section]
\newtheorem{proposition}[theorem]{Proposition}
\newtheorem{lemma}[theorem]{Lemma}

\theoremstyle{definition}
\newtheorem{definition}[theorem]{Definition}

\theoremstyle{remark}

\usepackage[textsize=tiny]{todonotes}

\icmltitlerunning{Contextual Linear Bandits with Delay as Payoff}
\newif\ifspacehack
\usepackage{natbib}
\hypersetup{
    colorlinks = blue,
    breaklinks,
    linkcolor = blue,
    citecolor = blue,
    urlcolor  = blue,
}
\usepackage{url} 
\usepackage{graphicx}
\usepackage{mathtools}
\usepackage{footnote}
\usepackage{float}
\usepackage{xspace}
\usepackage{multirow}
\usepackage{xcolor}
\usepackage{wrapfig}
\usepackage{framed}
\usepackage{bbm}
\usepackage[most]{tcolorbox}

\usepackage{footnote}
\usepackage{nicefrac}
\usepackage{makecell}
\usepackage[ruled,vlined]{algorithm2e}
\usepackage{amssymb}
\usepackage{bm}
\makesavenoteenv{tabular}
\makesavenoteenv{table}

\renewcommand{\tilde}{\widetilde}
\renewcommand{\hat}{\widehat}

\def \R {\mathbb{R}}

\newcommand{\calA}{{\mathcal{A}}}

\newcommand{\calX}{{\mathcal{X}}}
\newcommand{\calS}{{\mathcal{S}}}

\newcommand{\calI}{{\mathcal{I}}}

\newcommand{\calU}{{\mathcal{U}}}

\newcommand{\calT}{{\mathcal{T}}}
\newcommand{\calP}{{\mathcal{P}}}

\newcommand{\Reg}{{\mathrm{Reg}}}
\newcommand{\Alg}{{\mathsf{Alg}}}

\newcommand{\nctx}{\text{n-ctx}}

\newcommand{\E}{{\mathbb{E}}}

\newcommand{\gup}[1]{g^{(#1)}}

\newcommand{\inner}[1]{ \left\langle {#1} \right\rangle }

\newcommand{\order}{\mathcal{O}}

\newcommand{\term}[1]{\texttt{Term} ~(#1)\xspace}

\newcommand{\Err}[1]{\textsc{Err-Term}(#1)\xspace}
\newcommand{\regnctx}{\textsc{Reg-NCTX}\xspace}

\DeclareMathOperator*{\argmin}{argmin}
\DeclareMathOperator*{\argmax}{argmax}

\newcommand{\UCB}{{\operatorname{UCB}}}
\newcommand{\LCB}{{\operatorname{LCB}}}

\newcommand{\sgn}{\mbox{\sc sgn}}

\newcommand{\wh}{\widehat}

\newtheorem{event}{Event}

\newcommand{\otil}{\ensuremath{\tilde{\mathcal{O}}}}

\newcommand{\dist}{\calP}

\usepackage{lipsum,booktabs}
\usepackage{amsmath,mathrsfs,amssymb,amsfonts,bm,enumitem}
\usepackage{rotating}
\usepackage{pdflscape}
\usepackage{hyperref,url}
\hypersetup{
    colorlinks,
    breaklinks,
    linkcolor = blue,
    citecolor = blue,
    urlcolor  = blue,
}
\allowdisplaybreaks
\usepackage{appendix}
\usepackage{multirow,makecell}

\renewcommand{\tilde}{\widetilde}
\renewcommand{\hat}{\widehat}
\newcommand{\obs}{O}
\newcommand{\unobs}{E}
\newcommand{\unbiasSize}{c}
\newcommand{\unbias}{C}
\newcommand{\cnt}{k}

\def \E {\mathbb{E}}

\def \R {\mathbb{R}}

\usepackage{mathtools}
\newtcolorbox{promptbox}[1][]{
  colback=blue!5!white, colframe=blue!75!black,
  fonttitle=\bfseries, title=Prompt,
  left=2mm, right=2mm, top=2mm, bottom=2mm,
  boxrule=0.5mm,  % Thickness of the frame
  coltitle=black, % Color of the title text
  colbacktitle=blue!15!white, % Background color of the title
  breakable,      % Allows the box to break across pages
  #1
}

\usepackage{graphicx,color} % more modern

\definecolor{wine_red}{RGB}{228,48,64}
\definecolor{DSgray}{cmyk}{0,1,0,0}

\usepackage{prettyref}
\newcommand{\pref}[1]{\prettyref{#1}}

\newcommand{\savehyperref}[2]{\texorpdfstring{\hyperref[#1]{#2}}{#2}}
\newrefformat{eq}{\savehyperref{#1}{Eq. \textup{(\ref*{#1})}}}
\newrefformat{eqn}{\savehyperref{#1}{Eq.~(\ref*{#1})}}
\newrefformat{lem}{\savehyperref{#1}{Lemma~\ref*{#1}}}
\newrefformat{event}{\savehyperref{#1}{Event~\ref*{#1}}}
\newrefformat{def}{\savehyperref{#1}{Definition~\ref*{#1}}}
\newrefformat{line}{\savehyperref{#1}{Line~\ref*{#1}}}
\newrefformat{thm}{\savehyperref{#1}{Theorem~\ref*{#1}}}
\newrefformat{tab}{\savehyperref{#1}{Table~\ref*{#1}}}
\newrefformat{corr}{\savehyperref{#1}{Corollary~\ref*{#1}}}
\newrefformat{cor}{\savehyperref{#1}{Corollary~\ref*{#1}}}
\newrefformat{sec}{\savehyperref{#1}{Section~\ref*{#1}}}
\newrefformat{app}{\savehyperref{#1}{Appendix~\ref*{#1}}}
\newrefformat{assum}{\savehyperref{#1}{Assumption~\ref*{#1}}}
\newrefformat{asm}{\savehyperref{#1}{Assumption~\ref*{#1}}}
\newrefformat{ex}{\savehyperref{#1}{Example~\ref*{#1}}}
\newrefformat{fig}{\savehyperref{#1}{Figure~\ref*{#1}}}
\newrefformat{alg}{\savehyperref{#1}{Algorithm~\ref*{#1}}}
\newrefformat{rem}{\savehyperref{#1}{Remark~\ref*{#1}}}
\newrefformat{conj}{\savehyperref{#1}{Conjecture~\ref*{#1}}}
\newrefformat{prop}{\savehyperref{#1}{Proposition~\ref*{#1}}}
\newrefformat{proto}{\savehyperref{#1}{Protocol~\ref*{#1}}}
\newrefformat{prob}{\savehyperref{#1}{Problem~\ref*{#1}}}
\newrefformat{claim}{\savehyperref{#1}{Claim~\ref*{#1}}}
\newrefformat{que}{\savehyperref{#1}{Question~\ref*{#1}}}
\newrefformat{op}{\savehyperref{#1}{Open Problem~\ref*{#1}}}
\newrefformat{fn}{\savehyperref{#1}{Footnote~\ref*{#1}}}

\def \epsilon {\varepsilon}

\newcommand{\rad}{\mathsf{rad}}

\begin{document}

\twocolumn[
\icmltitle{Contextual Linear Bandits with Delay as Payoff}
\icmlsetsymbol{equal}{*}

\begin{icmlauthorlist}
\icmlauthor{Mengxiao Zhang}{a}
\icmlauthor{Yingfei Wang}{b}
\icmlauthor{Haipeng Luo}{c}
\end{icmlauthorlist}

\icmlaffiliation{a}{University of Iowa}
\icmlaffiliation{b}{University of Washington}
\icmlaffiliation{c}{University of Southern California}
%\icmlaffiliation{comp}{Company Name, Location, Country}
%\icmlaffiliation{sch}{School of ZZZ, Institute of WWW, Location, Country}

%\icmlcorrespondingauthor{Firstname2 Lastname2}{first2.last2@www.uk}

% You may provide any keywords that you
% find helpful for describing your paper; these are used to populate
% the "keywords" metadata in the PDF but will not be shown in the document
\icmlkeywords{Machine Learning, ICML}

\vskip 0.3in
]
\printAffiliationsAndNotice{}

\begin{abstract}
A recent work by \citet{schlisselberg2024delay} studies a delay-as-payoff model for stochastic multi-armed bandits, where the payoff (either loss or reward) is delayed for a period that is proportional to the payoff itself.
While this captures many real-world applications, the simple multi-armed bandit setting limits the practicality of their results.
In this paper, we address this limitation by studying the delay-as-payoff model for contextual linear bandits.
Specifically, we start from the case with a fixed action set and propose an efficient algorithm whose regret overhead compared to the standard no-delay case is at most
$D\Delta_{\max}\log T$, where $T$ is the total horizon, $D$ is the maximum delay, and $\Delta_{\max}$ is the maximum suboptimality gap. 
When payoff is loss, we also show further improvement of the bound, demonstrating a separation between reward and loss similar to \citet{schlisselberg2024delay}.
Contrary to standard linear bandit algorithms that construct least squares estimator and confidence ellipsoid, the main novelty of our algorithm is to apply a phased arm elimination procedure by only picking actions in a \emph{volumetric spanner} of the action set, which addresses challenges arising from both payoff-dependent delays and large action sets.
We further extend our results to the case with varying action sets by adopting the reduction from~\citet{hanna2023contexts}.  
Finally, we implement our algorithm and showcase its effectiveness and superior performance in experiments.
\end{abstract}

\section{Introduction}\label{sec: intro}
Stochastic multi-armed bandit (MAB)  is a well-studied theoretical framework for sequential decision making. In recent years, considerable investigation has been given to the realistic situations where the agent observes the payoff (either reward or loss) of an arm
only after a certain delayed period of time. However, most work assumes that the delays are {\it payoff-independent}. Namely, while the delay may depend on the chosen arm, it is sampled independently from the stochastic payoff of the chosen arm.

\citet{lancewicki2021stochastic} address this limitation by studying a setting where the delay and the reward are drawn together from a joint distribution.  
Later, \citet{tang2024stochastic} consider a special case where the delay is exactly the reward. Their motivation stems from response-adaptive clinical trials that aim at maximizing survival outcomes. For example, progression-free survival (PFS)—defined as the number of days after treatment until disease progression or death—is widely used to evaluate the effectiveness of a treatment. Notably, in this context, the ``delays'' in observing the PFS are the PFS itself. \citet{schlisselberg2024delay}  builds on and refines this investigation,  extending the study to the case where delay is the loss itself. Taken together, this delay-as-payoff framework effectively captures many real-world scenarios involving time-to-event data across many domains. For example, postoperative length of stay (PLOS) is one example of time-to-event data that  specifies the length of stay after surgery. Potential surgical procedures and postoperative care can be modeled as arms.  The delay—defined as the time until the patient is discharged—can be interpreted as the loss that we aim to minimize. 
As another example, in advertising, common metrics, including Average Time on Page (ATP) and Time to Re-engagement (that
tracks the time elapsed between a user’s initial interaction with an ad and subsequent engagement such as returning to the website), can be modeled as reward or loss inherently delayed by the same duration.

Despite such recent progress, the current consideration of {\it payoff-dependent} delay remains limited to the simple multi-armed bandit (MAB) setting. While MAB frameworks are foundational in decision-making problems, they have notable practical limitations. Specifically, they fail to account for the influence of covariates that drive heterogeneous responses across different actions. This makes them less suitable for scenarios involving a large number of  (potentially dynamically changing) actions and/or situations where context is crucial in shaping outcomes.

\paragraph{Contributions.} 
Motivated by this limitation, in this work, we extend the delay-as-payoff model from MAB to contextual linear bandits, a practical framework that is widely used in real-world applications.  Specifically, our contributions are as follows.
\begin{itemize}[leftmargin=*]
    \item As a first step, in \pref{sec: linear}, we study stochastic linear bandits with a fixed action set (known as the non-contextual setting).
    We point out the difficulty of directly combining the standard \texttt{LinUCB} algorithm with the idea of~\citet{schlisselberg2024delay},
    and propose a novel phased arm elimination algorithm that  only selects actions from a \emph{volumetric spanner} of the action set.
    In the delay-as-loss case, we prove that, compared to the standard regret in the delay-free setting, the overhead caused by the payoff-dependent delay for our algorithm is 
    $\order(\min\{nd^\star
    \log(T/n)+D\Delta_{\max}, D\Delta_{\max}\log(T/n)\})$, where $n$ is the dimension of the action set, $T$ is the total horizon, $\Delta_{\max}$ is the maximum suboptimality gap,
    $d^\star$ is the expected delay of the optimal action,
    and $D$ is the maximum possible delay
    (formal definitions are deferred to \pref{sec: pre}). 
    This instance-dependent bound is in the same spirit as the one of~\citet{schlisselberg2024delay} and is small whenever the optimal action has a small loss.
    In the delay-as-reward case, a slightly worse bound is provided in \pref{app: reward}; such a separation between loss and reward is similar to the results of~\citet{schlisselberg2024delay}.
    
    \item Next, in \pref{sec: contextual}, we extend our results to the contextual case where the action set is varying and drawn i.i.d. from an unknown distribution. 
    Using a variant of our non-contextual algorithm (that can handle misspecification) as a subroutine,    
    we apply the contextual to non-contextual reduction recently proposed by \citep{hanna2023contexts} and show that the resulting algorithm enjoys a similar regret guarantee despite having varying action sets, establishing the first regret guarantee for contextual linear bandits with delay-as-payoff.

    \item In \pref{sec: experiment}, we implement our algorithm and test it on synthetic linear bandits instances, demonstrating its superior performance against a baseline that runs \texttt{LinUCB} with only the currently available feedback.
\end{itemize}

\paragraph{Related works.} 
Recent research has investigated different problems of learning under bandit feedback with delayed payoff, addressing various new challenges caused by the combination of delay and bandit feedback. As mentioned, most studies assume {\it payoff-independent} delays. Among this line of research, \citet{Dudk2011EfficientOL} is among the first to consider delays in stochastic MAB with a constant delay. \citet{mandel2015queue} and \citet{joulani2013online} extend the consideration to stochastic delay, with the assumption that the delay is bounded.

Subsequent studies on  i.i.d. stochastic delays differentiate between {\it arm-independent} and {\it arm-dependent} delays. For {\it arm-independent} delays, \citet{zhou2019learning,vernade2020linear,blanchet2024delay-contextual-linear} shows regret characterizations for (generalized) linear stochastic contextual bandits.
\citet{pike2018bandits} considers  aggregated anonymous feedback, under the assumption that the expected delay is bounded and known to the learner. {\it Arm-dependent} stochastic delays have been explored in various settings, including \citet{gael2020stochastic,arya2020randomized,lancewicki2021stochastic}.

Far less attention has been given to {\it payoff-dependent} stochastic delays. The setting in
\citet{vernade:hal-01545667} implies a dependency between the reward and the delay, as a current non-conversion could be the result of a delayed reward of 1.  \citet{lancewicki2021stochastic} considers the case where the stochastic delay in each round and the reward are drawn from a general joint distribution. 
\citet{tang2024stochastic} investigates strongly reward-dependent delays, specifically motivated by medical settings where the delay is equal to the reward. \citet{ schlisselberg2024delay} follows this investigation and extends the discussion to  delay as loss, and provides a tighter regret bound. Although with a slightly different focus, \citet{thune2019nonstochastic, zimmert2020optimal,gyorgy2021adapting, van2022nonstochastic,van2023unified} and several other works study non-stochastic bandits, where  both the delay and rewards are adversarial. 

Nevertheless, the {\it payoff-dependent} (either loss or reward) delays are only studied under stochastic multi-armed bandits (MAB).  In this work, we extend the study to contextual linear bandits, significantly broadening its practicality.

\section{Preliminary}\label{sec: pre}
Throughout this paper, we use $[N]$ to denote $\{1,2,\dots,N\}$ for some positive integer $N$. Let $\R^n_+$ be the $n$-dimensional Euclidean space in the positive orthant and $\mathbb{B}_2^n(1)=\{v \in \R^n:~\|v\|_2\leq 1\}$ be the $n$-dimensional unit ball with respect to $\ell_2$ norm. Define $\calU[a,b]$ to be the uniform distribution over $[a,b]$. For a real number $a$, define $\sgn(a)$ as the sign of $a$. For two real numbers $a$ and $b$, define $a\vee b\triangleq \max\{a,b\}$. For a finite set $\calS$, denote $|\calS|$ as the cardinality of $\calS$. The notation $\otil(\cdot)$ hides all logarithmic dependencies.

In this paper, we consider the delay-as-payoff model proposed by \citet{schlisselberg2024delay}, in which the delay of the payoff is proportional to the payoff itself. Specifically, we study stochastic linear bandits in this model, and we start with a fixed action set as the first step (referred to as the non-contextual case) and then move on to the case with a time-varying action set (referred to as the contextual case). For conciseness, we mainly discuss the payoff-as-loss case, but our algorithm and results can be directly extended to the payoff-as-reward case (see \pref{app: reward}).

\textbf{Non-contextual stochastic linear bandits.} In this problem, the learner is first given a \emph{fixed} finite set of actions $\calA\subset \R_+^n\cap\mathbb{B}_2^n(1)$ with cardinality $|\calA|=K$. Let $D>0$ be the maximum possible delay. At each round $t\in [T]$, the learner selects an action $a_t\in \calA$ and incurs a loss $u_t=\mu_{a_t}+\eta_t\in [0,1]$ where $\eta_t$ is zero-mean random noise, $\mu_{a}=\inner{a,\theta}$ is the expected payoff of action $a$, and $\theta\in \R_+^n\cap\mathbb{B}_2^n(1)$ is the model parameter that is unknown to the learner.\footnote{We enforce both $\calA\subset \R_+^n\cap\mathbb{B}_2^n(1)$ and $\theta\in \R_+^n\cap\mathbb{B}_2^n(1)$ to make sure that the payoff (and hence the delay) is non-negative.}  Then, the loss is received by the learner at the end of round $\lceil t+d_t \rceil$ where the delay $d_t=D\cdot u_t$ (that is, proportional to the loss).  
The goal of the learner is to minimize the (expected) pseudo regret defined as follows:
\begin{align}\label{eqn:reg_linear_reward}
    \Reg \triangleq \E\left[\sum_{t=1}^T\inner{a_t,\theta}\right] - T\cdot \min_{a\in \calA}\inner{a,\theta}.
\end{align}
Let $a^\star \in \argmin_{a\in \calA}\inner{a,\theta}$ be an optimal action, $\mu^\star = \mu_{a^\star}$ be its expected loss,
and $d^\star = D\mu^\star$ be its expected delay.
For an action $a$, define $\Delta_a = \inner{a-a^\star,\theta}$ as its sub-optimality gap. Further define $\Delta_{\min} = \min_{a\in \calA, \Delta_a>0}\Delta_a$ and $\Delta_{\max} = \max_{a\in \calA}\Delta_a$ to be the minimum and maximum non-zero sub-optimality gap respectively. 

We point out that the standard multi-armed bandit (MAB) setting considered in~\citet{schlisselberg2024delay} is a special case of our setting with $\calA$ being the set of all standard basis vectors in $\R^n$.

\textbf{Contextual stochastic linear bandits.} In the contextual case, the main difference is that the action set is not fixed but \emph{changing over rounds}. Specifically, at each round $t$, the learner first receives an action set $\calA_t\subset \R_+^n\cap \mathbb{B}_2^n(1)$ (which can be seen as a context), where we assume that $\calA_t$ is drawn i.i.d. from an unknown distribution $\dist$. The rest of the protocol remains the same, and the goal of the learner is still to minimize the (expected) pseudo regret, defined as:
\begin{align*}
    \Reg\triangleq \E
    \left[\sum_{t=1}^T\inner{a_t,\theta} - \sum_{t=1}^T\min_{a_t^\star\in\calA_t}\inner{a_t^\star,\theta}\right],
\end{align*}
where the expectation is taken over both the internal randomness of the algorithm and the external randomness in the action sets and loss noises.

\section{First Step: Non-Contextual Linear Bandits}\label{sec: linear}

In this section, we focus on the non-contextual case, which serves as a building block for eventually solving the contextual case. Before introducing our algorithm, we first briefly introduce the successive arm elimination algorithm for the simpler MAB setting proposed by \citet{schlisselberg2024delay} and their ideas of handling payoff-dependent delay. Specifically, their algorithm starts with a guess $B=1/D$ on the optimal action's loss, and maintains an active set of arms. The algorithm pulls each arm in the active set once, and constructs two LCB's (lower confidence bound) and one UCB (upper confidence bound) for each action in the active set as follows (supposing the current round being $t$):
\begin{align}
        \LCB_{t,1}(a) &= \frac{1}{\cnt_t(a)}\sum_{\tau\in\obs_t(a)}u_\tau - \sqrt{\frac{2\log T}{\cnt_t(a)}}, \label{eqn:lcb-1-mab}\\
        \LCB_{t,2}(a) &= \frac{1}{\unbiasSize_t(a)}\sum_{\tau\in\unbias_t(a)}u_{\tau} - \sqrt{\frac{2\log T}{\unbiasSize_t(a)\vee 1}}, \label{eqn:lcb-2-mab}\\
        \UCB_{t}(a) &= \frac{1}{\unbiasSize_t(a)}\sum_{\tau\in\unbias_t(a)}u_{\tau} + \sqrt{\frac{2\log T}{\unbiasSize_t(a)\vee 1}},\label{eqn:ucb-mab}
\end{align}
where $\cnt_t(a) = \sum_{\tau=1}^t\mathbbm{1}\{a_t=a\}$ is the total number of pulls of action $a$ till round $t$, $\obs_t(a) = \{\tau: \tau+d_{\tau}\leq t \text{~and~} a_{\tau}=a\}$ is the set of rounds where action $a$ is chosen and its loss has been received by the end of round $t$, $\unbias_t(a) = \{\tau \leq t-D: a_\tau = a\}$ is the set of rounds up to $t-D$ where action $a$ is chosen (so its loss has for sure been received by the end of round $t$), 
and $\unbiasSize_t(a)=|\unbias_t(a)|$. Specifically, \pref{eqn:lcb-1-mab} constructs an LCB of action $a$ assuming all the action's unobserved loss to be $0$ (the smallest possible), while \pref{eqn:lcb-2-mab} and \pref{eqn:ucb-mab} construct an LCB and a UCB using only the losses no later than round $t-D$ (which must have been received by round $t$), making the empirical average a better estimate of the expected loss. With $\UCB_t(a)$ and $\LCB_t(a) = \max\{\LCB_{t,1}(a), \LCB_{t,2}(a)\}$ constructed, the algorithm eliminates an action $a$ if its $\LCB_t(a)$ is larger than $\min\{\UCB_t(a'),B\}$ for some $a'$ in the active set. If all the actions are eliminated, this means that the guess $B$ on the optimal loss is too small, and the algorithm starts a new epoch with $B$ doubled.\footnote{In fact, \citet{schlisselberg2024delay} construct yet another LCB based on the number of unobserved losses. We omit this detail since we are not able to use this to further improve our bounds for linear bandits.}

\paragraph{Challenges} However, this approach cannot be directly applied to linear bandits. Specifically, standard algorithms for stochastic linear bandits without delay (e.g., \citet{li2010contextual,abbasi2011improved}) all construct  UCB/LCB for each action by constructing an ellipsoidal confidence set for $\theta$. In the delay-as-payoff model, while it is still viable to construct UCB/LCB similar to \pref{eqn:lcb-2-mab} and \pref{eqn:ucb-mab} via a standard confidence set of $\theta$, it is difficult to construct an LCB counterpart similar to \pref{eqn:lcb-1-mab}.
This is because one action's loss is estimated using observations of all other actions in linear bandits, and naively treating the unobserved loss of one action as zero might not necessarily lead to an underestimation of another action. 

\paragraph{Our ideas} To bypass this barrier, we give up on estimating $\theta$ itself and propose to construct UCB/LCB for each action using the observed losses of the \emph{volumatric spanner} of the action set. A volumetric spanner of an action set $\calA$ is defined such that every action in $\calA$ can be expressed as a linear combination of the spanner. 

\begin{definition}[Volumetric Spanner~\citep{hazan2016volumetric}]\label{def:volume}
Suppose that $\calA = \{a_1, a_2, \dots , a_N\}$ is a set of vectors in $\R^n$. We say $\calS\subseteq \calA$ is a \emph{volumetric spanner} of $\calA$ if for any $a\in \calA$, we can write it as $a=\sum_{b\in \calS}\lambda_b\cdot b$ for some $\lambda\in \R^{|\calS|}$ with $\|\lambda\|_2\leq 1$. 
\end{definition}

Due to the linear structure, it is clear that the loss $\mu_a$ of action $a$ can be decomposed in a similar way as $\sum_{b\in \calS}\lambda_b \mu_b$,
making it possible to estimate every action's loss by only estimating the loss of the spanner.
Moreover, such a spanner can be efficiently computed:
\begin{proposition}[\citet{bhaskara2023tight}]\label{prop:volume}
Given a finite set $\calA$ of size $K$, there exists an efficient algorithm finding a volumetric spanner $\calS$ of $\calA$ with $|\calS|=3n$ within $\order(Kn^3\log n)$ runtime.
\end{proposition}

\setcounter{AlgoLine}{0}
\begin{algorithm}
\caption{Phased Elimination via Volumetric Spanner for Linear Bandits with Delay-as-Loss}\label{alg:lossLB}

\nl Input: maximum possible delay $D$, action set $\calA$, $\beta>0$. 

\nl Initialization: optimal loss guess $B=1/D$.

\nl Initialization: active action set $\calA_1=\calA$. \label{line: restart}

 \For{$m=1,2,\dots,$}{
    \nl Find $\calS_m=\{a_{m,1},\dots,a_{m,|\calS_m|}\}$, a volumetric spanner of $\calA_m$ with $|\calS_m|= 3n$. \label{line:volume}
    
    \nl Pick each $a\in \calS_m$ $2^m$ times in a round-robin way. \label{line:round-robin}

    \nl Let $\calI_m$ contain all the rounds in this epoch.
    
    \nl For each $a\in \calS_m$, calculate the following quantities: \label{line:spanner-ucb-lcb}
    {\small
    \begin{align}
        &\hat{\mu}_{m}^+(a)=\frac{1}{2^m}\Big(\sum_{\tau\in \obs_m(a)}u_{\tau} + \sum_{\tau\in \unobs_m(a)}1\Big), \label{eqn:mean-up}\\
        &\hat{\mu}_{m}^-(a)=\frac{1}{2^m}\sum_{\tau\in \obs_m(a)}u_{\tau}, \label{eqn:mean-low}\\
        &\hat{\mu}_{m,1}^{+}(a)=\hat{\mu}^{+}_{m}(a)+\frac{\beta}{2^{m/2}}\|a\|_2, \label{eqn:loss-ucb-linear-1}\\
        &\hat{\mu}_{m,1}^{-}(a)=\hat{\mu}^{-}_{m}(a)-\frac{\beta}{2^{m/2}}\|a\|_2,\label{eqn:loss-lcb-linear-1}\\
        &\hat{\mu}_m^{F}(a)=\frac{1}{\unbiasSize_m(a)}\sum_{\tau\in \unbias_m(a)}u_{\tau}, \label{eqn:mean_unbiased}\\
        &\hat{\mu}_{m,2}^{+}(a)=\hat{\mu}_m^F(a)+\frac{\beta}{\sqrt{\unbiasSize_m(a)}}\|a\|_2, \label{eqn:loss-ucb-linear-2}\\
        &\hat{\mu}_{m,2}^{-}(a)=\hat{\mu}_m^F(a)-\frac{\beta}{\sqrt{\unbiasSize_m(a)}}\|a\|_2, \label{eqn:loss-lcb-linear-2}
    \end{align}
    }
    where $\unbiasSize_m(a) = |\unbias_m(a)|$, $\unbias_m(a) = \{\tau\in \calI_m: \tau+D\in\calI_m, a_{\tau}=a\}$, $\obs_m(a) = \{\tau\in \calI_m: \tau+d_{\tau}\in\calI_m, a_{\tau}=a\}$, and
    $\unobs_m(a)= \{\tau\in \calI_m: a_{\tau}=a\}\setminus\obs_m(a)$.

    \For{each $a\in \calA_m$}{
        \nl \label{line: decompose}
        Decompose $a$ as $a=\sum_{i=1}^{|S_m|}\lambda_{m,i}^{(a)}a_{m,i}$ with $\|\lambda_{m}^{(a)}\|_2\leq 1$ and calculate 
        {\small
        \begin{align}
            &\UCB_{m}(a)=\sum_{i=1}^{|\calS_m|}\lambda_{m,i}^{(a)}\cdot\hat{\mu}_{m,2}^{\sgn(\lambda_{m,i}^{(a)})}(a_{m,i}), \label{eqn:loss-ucb-f-all-action} \\
            &\LCB_m(a) = \max_{j\in \{1,2\}}\{\LCB_{m,j}(a)\} \;\;\text{where} \nonumber  \\
            & \LCB_{m,j}(a)=\sum_{i=1}^{|\calS_m|}\lambda_{m,i}^{(a)}\cdot\hat{\mu}_{m,j}^{\sgn(-\lambda_{m,i}^{(a)})}(a_{m,i}),\label{eqn:loss-lcb-all-action}
        \end{align}
        }
    }
    
    \nl Set $\calA_{m+1} = \calA_m$.
    
    \For{$a\in \calA_m$}{
        \nl \label{line:eliminate}  
        \If{$\exists a'\in \calA_m$, s.t. $\LCB_m(a) \geq \min\{\UCB_m(a'),B\} $}
        {
          Eliminate $a$ from $\calA_{m+1}$.
        }
    }
    \nl \If{$\calA_{m+1}=\emptyset$}{
        Set $B\leftarrow 2B$ and go to \pref{line: restart}.
    }
}
\end{algorithm}

Equipped with the concept of volumetric spanner, we are now ready to introduce our algorithm (see \pref{alg:lossLB}). 
Specifically, our algorithm also makes a guess $B$ on the loss of the optimal action. 
With this guess, it proceeds to multiple epochs of arm elimination procedures, with the active action set initialized as $\calA_1 = \calA$.
In each epoch $m$, instead of picking every action in the active set $\calA_m$, we first compute a volumetric spanner $\calS_m$ of $\calA_m$ with $|\calS_m|=3n$ (\pref{line:volume}), which can be done efficiently according to \pref{prop:volume}, 
and then pick each action in the spanner set $\calS_m$ for $2^m$ rounds in a round-robin way (\pref{line:round-robin}).

After that, we calculate two UCBs and two LCBs for actions in the spanner, in a way similar to the simpler MAB setting discussed earlier (\pref{line:spanner-ucb-lcb}).
Specifically, 
the first one is in the same spirit of \pref{eqn:lcb-1-mab}:
we calculate $\hat{\mu}_m^+(a)$ ( $\hat{\mu}_m^-(a)$) as an overestimation (underestimation) of the expected loss of action $a$ by averaging over all observed losses from the rounds in $\obs_m(a)$ as well as the maximum (minimum) possible value of unobserved losses from the rounds in $\unobs_m(a)$; see \pref{eqn:mean-up} and \pref{eqn:mean-low}.
The first UCB (LCB) $\hat{\mu}_{m,1}^+(a)$ ($\hat{\mu}_{m,1}^-(a)$) is then computed based on $\hat{\mu}_m^+(a)$ ($\hat{\mu}_m^-(a)$) by incorporating a standard confidence width $\frac{\beta}{\sqrt{2^m}}\|a\|_2$ for some coefficient $\beta$; see \pref{eqn:loss-ucb-linear-1} and \pref{eqn:loss-lcb-linear-1}.
Then, to compute the second UCB/LCB, which is in the same spirit as \pref{eqn:lcb-2-mab} and \pref{eqn:ucb-mab}, we calculate an unbiased estimation $\hat{\mu}_m^F(a)$ of the expected loss of $a$ by averaging losses from the rounds in $\unbias_m(a)$, that is, all the rounds where the observation must have been revealed; see \pref{eqn:mean_unbiased}.
Note that the number of such rounds, $\unbiasSize_m(a) = |\unbias_m(a)|$, is a fixed number, and thus $\hat{\mu}_m^F(a)$ is indeed unbiased.
Similarly, we incorporate a standard confidence width $\frac{\beta}{\sqrt{c_m(a)}}\|a\|_2$ to arrive at the second UCB $\hat{\mu}_{m,2}^+(a)$ and LCB $\hat{\mu}_{m,2}^-(a)$; see \pref{eqn:loss-ucb-linear-2} and \pref{eqn:loss-lcb-linear-2}.

The next step of our algorithm is to use these UCBs/LCBs for the spanner to compute corresponding UCBs/LCBs for every active action in $\calA_m$ (\pref{line: decompose}). Specifically, for each action $a\in \calA_m$, according to the definition of a volumetric spanner (\pref{def:volume}), we can write $a$ as a linear combination of actions in $\calS_m$: $\sum_{i=1}^{|S_m|}\lambda_{m,i}^{(a)}a_{m,i}$. As mentioned, due to the linear structure of losses, we also have $\mu_a = \sum_{i=1}^{|S_m|}\lambda_{m,i}^{(a)}\mu_{a_{m,i}}$.
Thus, when constructing a UCB (or similarly LCB) for $a$, based on whether $\lambda_{m,i}^{(a)}$ is positive or not, we decide whether to use the UCB or LCB of $a_{m,i}$; see \pref{eqn:loss-ucb-f-all-action}, a counterpart of \pref{eqn:ucb-mab}, and \pref{eqn:loss-lcb-all-action}, a counterpart of \pref{eqn:lcb-1-mab} and \pref{eqn:lcb-2-mab}.\footnote{This also explains why we need $\hat{\mu}_m^+(a)$, a quantity not used in~\citet{schlisselberg2024delay}.}

At the end of an epoch, we eliminate all actions from the active action set if their LCB is either larger than the UCB of certain action or the guess $B$ on the optimal loss  (\pref{line:eliminate}). 
If the active set becomes empty, this means that the guess $B$ is too small, and the algorithm restarts with the guess doubled; 
otherwise, we continue to the next epoch.

\paragraph{Theoretical performance}
We prove the following regret bound for our algorithm. 
\begin{restatable}{theorem}{lossLB}
\label{thm:main-non-contextual}
    \pref{alg:lossLB} with $\beta=\sqrt{2\log(KT^3)}$ guarantees: 
\begin{align*}
        \Reg &\leq \order\left(\min\left\{V_1,V_2\right\}\right) + \log(d^\star)\cdot \order\left( \min\left\{W_1,W_2\right\}\right),
    \end{align*}
    where $V_1=\frac{n^2\log(KT)\log(T/n)\log(d^\star)}{\Delta_{\min}}$, $V_2=n\sqrt{T\log(d^\star)\log(KT)}$, $W_1=nd^\star\log (T/n)+D\Delta_{\max}$, and $W_2=D\Delta_{\max}\log (T/n)$. 
\end{restatable}
The first term in the regret bound $\order\left(\min\left\{V_1,V_2\right\}\right)$ is of order $\otil(\min\{\frac{n^2}{\Delta_{\min}}, n\sqrt{T}\})$, which matches the standard regret bound of LinUCB in the case without delay~\citep{abbasi2011improved}.
The second term is the overhead caused by delay and is in the same spirit as the result of~\citet{schlisselberg2024delay}:
focusing only on the part that grows in $T$, 
we see that $W_1$ only depends on $d^\star$, the expected delay of the optimal action (and hence the smallest delay among all actions),
while $W_2$ depends on the maximum possible delay $D$ but scaled by $\Delta_{\max}$, the largest sub-optimality gap.
Therefore, the delay overhead of our algorithm is small when either the shortest delay is small or all actions have similar losses.
We remark again that in the delay-as-reward setting, we obtain similar results; see \pref{app: reward} for details.

\subsection{Analysis}\label{sec: alg}
In this section, we provide a proof sketch of \pref{thm:main-non-contextual}. Detailed proofs are deferred to \pref{app:loss}.

The proof starts by proving that $\UCB_m(a)$ and $\LCB_m(a)$ are indeed valid UCB and LCB respectively for all actions in $\calA_m$. 
This follows from first using standard concentration inequalities to show that $\hat{\mu}_{m,1}^+(a)$ and $\hat{\mu}_{m,2}^+(a)$ ($\hat{\mu}_{m,1}^-(a)$ and $\hat{\mu}_{m,2}^-(a)$) are valid UCBs (LCBs) for each action in the spanner, 
and then generalizing it to every action $a \in \calA_m$ according to its decomposition over the actions in the spanner.

With this property, our analysis then proceeds to control the regret of \pref{alg:lossLB} for each guess of $B$ separately. Let $\calT_B$ be the set of rounds when \pref{alg:lossLB} runs with guess $B$. 
In \textbf{Step 1}, we first show that the use of $\LCB_{m,2}(a)$ and $\UCB_m(a)$ ensures a regret bound of $\order\left(\min\{R_1,R_2\}+D\Delta_{\max}\log(T/n)\right)$ where $R_1=\frac{n^2\log(KT)\log(T/n)}{\Delta_{\min}}$ and $R_2=n\sqrt{|\calT_B|\log(KT)}$,
and then in \textbf{Step 2}, we show that the use of $\LCB_{m,1}(a)$ and $\UCB_m(a)$ ensures a regret bound of
$\order(\min\{R_1,R_2\}+(nd^\star+DB)\log(T/n)+D\Delta_{\max})$.

\paragraph{Step 1}
For notational convenience, we define 
\begin{align*}
    \rad_{m,a}^F=\beta\sum_{i=1}^{|\calS_m|}|\lambda_{m,i}^{(a)}|\cdot\frac{\|a\|_2}{\sqrt{\unbiasSize_m(a_{m,i})}}
\end{align*}
to be the total confidence radius of action $a$ coming from the definition of $\LCB_{m,2}(a)$ and $\UCB_m(a)$. 
Via a standard analysis of arm elimination, 
we show that that if an action $a$ is not eliminated at the end of epoch $m$, we have
\begin{align*}
    \Delta_a \leq 4\max_{a\in\calA_m}\rad_{m,a}^F \leq \frac{4\sqrt{3n}\beta}{\min_{a_m\in\calS_m}\sqrt{\unbiasSize_m(a_m)}},
\end{align*}
where the second inequality uses Cauchy-Schwarz inequality and the properties of volumetric spanners, specifically that $\|\lambda_{m}^{(a)}\|_2\leq 1$ and $|\calS_m|=3n$. To provide a lower bound on $c_m(a')$ for any $a'\in\calS_m$, note that we pick each action $a'\in \calS_m$ $2^m$ times in a round-robin manner, and thus
\begin{align*}
    c_m(a') \geq 2^m - \frac{D}{|\calS_m|}-1 = 2^m - \frac{D}{3n}-1.
\end{align*}
Rearranging the terms, we then obtain
\begin{align}\label{eqn:epoch_bound_1}
    2^m\Delta_a \leq \frac{48n\beta^2}{\Delta_a} + \frac{D\Delta_a}{3n} + \Delta_a.
\end{align}
Taking summation over all $a\in\calS_m$ and $m$, and noticing that the total number of epochs is bounded by $M=\lceil\log_2(|\calT_B|/3n)\rceil$, we arrive at the following $\order(R_1+D\Delta_{\max}\log(T/n))$ regret guarantee:
\begin{align}
&\sum_{m=1}^{M}\sum_{a\in\calS_m}2^m\Delta_a \nonumber\\
    &\leq \sum_{m=1}^{M}\sum_{a\in\calS_m,\Delta_a>0}2\cdot\left(\frac{48n\beta^2}{\Delta_a}+\frac{D\Delta_a}{3n}+\Delta_a\right) \nonumber\\
    &\leq \sum_{m=1}^{M}\sum_{a\in\calS_m,\Delta_a>0}\order\left(\frac{n\log (KT)}{\Delta_a}\right) + \order\left(D\Delta_{\max}\log(T/n)\right),\nonumber \\
    &\leq \order\left(\frac{n^2\log(T/n)\log (KT)}{\Delta_{\min}}\right) + \order\left(D\Delta_{\max}\log(T/n)\right),\nonumber
\end{align}
where the first inequality is because $a\in\calS_m$ is not eliminated in epoch $m-1$ and the last inequality is by lower bounding $\Delta_a$ by $\Delta_{\min}$.

To obtain the other instance-independent regret bound $\order(R_2+D\Delta_{\max}\log(T/n))$, we bound the regret differently by considering $\Delta_a\geq \beta\sqrt{n/2^m}$ and $\Delta_a\leq \beta\sqrt{n/2^m}$ separately:
\begin{align}
    &\sum_{m=1}^{M}\sum_{a\in\calS_m}2^m\Delta_a \nonumber \\
    &\leq \sum_{m=1}^{M}\sum_{a\in\calS_m,\Delta_a\geq\beta\sqrt{n/2^m}}\left(\frac{512n\beta^2}{\Delta_a}+\frac{2D\Delta_a}{3n}+2\Delta_a\right) \nonumber\\
    &\qquad +\sum_{m=1}^{M}\sum_{a\in\calS_m,\Delta_a\leq\beta\sqrt{n/2^m}}\left(2^m\Delta_a\right) \nonumber \\
    &\leq \order(n\sqrt{|\calT_B|\log(KT)} + D\Delta_{\max}\log(T/n))\nonumber.
\end{align}

\paragraph{Step 2}
To obtain the other regret bound $\order(\min\{R_1,R_2\}+(nd^\star+DB)\log(T/n)+D\Delta_{\max})$ with a different delay overhead, we similarly define
\begin{align*}
    \rad_{m,a}^{N} &= \beta\sum_{i=1}^{|\calS_m|}|\lambda_{m,i}^{(a)}|\cdot \frac{\|a\|_2}{\sqrt{2^m}}
\end{align*}
as the total confidence radius of action $a$ coming from the definition of $\LCB_{m,1}(a)$. 
Further let $\wh{\mu}_m(a) = \frac{1}{2^m}\left(\sum_{\tau\in \obs_m(a)\cup\unobs_m(a)}u_{\tau}\right)$ be the empirical mean of action $a$'s loss within epoch $m$ (which is generally not available to the algorithm due to delay). According to the construction of $\wh{\mu}_m^{+}(a)$ and $\wh{\mu}_m^{-}(a)$, we know that for all $a\in\calS_m$,
\begin{align*}
    \wh{\mu}_m^{+}(a)\leq \wh{\mu}_m(a) + \frac{|\unobs_m(a)|}{2^m},~~\wh{\mu}_m^{-}(a)\geq \wh{\mu}_m(a) - \frac{|\unobs_m(a)|}{2^m}.
\end{align*}
Then, for any action $a\in\calA_m$ that is not eliminated at the end of epoch $m$, using the fact that $a=\sum_{i=1}^{|\calS_m|}\lambda_{m,i}^{(a)}a_{m,i}$, we obtain with high probability:
\begin{align}
    \mu_a &\leq \sum_{i=1}^{|\calS_m|}\lambda_{m,i}^{(a)}\cdot \hat{\mu}_{m}(a_{m,i}) + \rad_{m,a}^{N} \nonumber\\
    &\leq \LCB_{m,1}(a) + \rad_{m,a}^{N}+\sum_{i=1}^{|\calS_m|}|\lambda_{m,i}^{(a)}|\cdot \frac{|\unobs_m(a_{m,i})|}{2^m} \nonumber\\
    &\leq \LCB_{m,1}(a) +\rad_{m,a}^{N} \nonumber\\
    &\qquad + \sum_{i=1}^{|\calS_m|}|\lambda_{m,i}^{(a)}|\cdot\left(\frac{2D\mu_{a_{m,i}}}{2^m|\calS_m|}+\frac{16\log KT +2}{2^m}\right) \\
    &\leq B +\rad_{m,a}^{N} \nonumber\\
    &\qquad + \sum_{i=1}^{|\calS_m|}|\lambda_{m,i}^{(a)}|\cdot\left(\frac{2D\mu_{a_{m,i}}}{2^m|\calS_m|}+\frac{16\log KT +2}{2^m}\right),\label{eqn:small-loss}
\end{align}
where the first inequality is by standard Azuma-Hoeffding's inequality, the third inequality is by Lemma C.2 of \citet{schlisselberg2024delay} (included as \pref{lem:high-prob-event} in the appendix for completeness), and the last inequality is because $a$ is not eliminated at the end of epoch $m$.

\setcounter{AlgoLine}{0}
\begin{algorithm*}[htbp]
\caption{Reduction from Contextual Linear Bandits to Non-Contextual Linear Bandits~\citep{hanna2023contexts}}\label{alg:reduction}
Input: confidence level $\delta$, an instance $\Alg_{\nctx}$ of \pref{alg:lossLBmis} with $\beta=\sqrt{2\log(KT^3)}$. 

Let $\Theta'$ be a $\frac{1}{T}$-cover of $\Theta$ with size $\order(T^n)$.

\For{$m=1,2,\dots$}{
    \nl Construct action set $\calX_{m}=\{\gup{m}(\theta)~\vert~\theta\in \Theta'\}$ where   $\gup{m}(\theta)=\frac{1}{2^{m-1}}\sum_{\tau=1}^{2^{m-1}}\argmin_{a\in \calA_\tau}\inner{a,\theta}$.

    \nl Initiate $\Alg_{\nctx}$ with action set $\calX_m$ and misspecification level $\epsilon_m=\min\{1,2\sqrt{\log(T|\Theta'|/\delta)/2^m}\}$. \label{line:misspecific_level}
    
    \nl \For{$t=2^{m-1}+1,\dots,2^m$}{
        \nl $\Alg_{\nctx}$ outputs action $\gup{m}(\theta_t)$.

        \nl Observe $\calA_t$ and select $a_t=\argmin_{a\in \calA_t}\inner{a,\theta_t}$.

        \nl Observe the loss $u_\tau$ for all $\tau$ such that $\tau+d_{\tau}\in (t-1,t]$ and send them to $\Alg_{\nctx}$.
    }
    
}
\end{algorithm*}

Now consider two cases. When $B\geq \frac{\mu_a}{2}$, we know that $\Delta_a\leq \mu_a - \mu^\star\leq 2B$. Using the previous \pref{eqn:epoch_bound_1}, we know that
\begin{align}\label{eqn:small-loss-1}
    2^m\Delta_a\leq \order\left(\frac{n\beta^2}{\Delta_a}+\frac{DB}{n}\right).
\end{align}
Otherwise, when $B < \frac{\mu_a}{2}$, with some manipulation on \pref{eqn:small-loss}, we show that
\begin{align}\label{eqn:small-loss-2}
    2^m\Delta_a\leq \order\left(\frac{n\beta^2}{\Delta_a}+\frac{\sum_{i=1}^{|\calS_m|}D\mu_{a_{m,i}}}{n}\right).
\end{align}
Combining \pref{eqn:small-loss-1} and \pref{eqn:small-loss-2}, we then obtain that within epoch $m$, the regret is bounded by
\begin{align}\label{eqn:small-loss-3}
\order\left(\sum_{a\in\calS_m}\frac{n\beta^2}{\Delta_a}+DB+D\sum_{i=1}^{|\calS_{m-1}|}\mu_{a_{m-1,i}}\right),
\end{align} 
since all active actions in epoch $m$ are not eliminated in epoch $m-1$.
The first term $\sum_{a\in\calS_m}\frac{n\beta^2}{\Delta_a}$ in \pref{eqn:small-loss-3} eventually leads to the $\min\{R_1,R_2\}$ term in the claimed regret bound, by the exact same reasoning as in \textbf{Step 1}.
The second term explains the final $DB\log(T/n)$ term in the regret bound (recall that number of epoch is of order $\order(\log(T/n))$).
Finally, the last term in \pref{eqn:small-loss-3} can be written as
$D\sum_{i=1}^{|\calS_{m-1}|} \Delta_{a_{m-1,i}} + 3n\cdot d^\star$,
and the term $D\sum_{i=1}^{|\calS_{m-1}|} \Delta_{a_{m-1,i}}$ is one half of the regret incurred in epoch $m-1$ as long as $2^{m-1}>2D$ (otherwise, the epoch length is smaller than $D$, and we bound the regret trivially by $D\Delta_{\max}$).
Summing over all epochs and rearranging the terms thus leads to the a term $nd^\star\log(T/n)$ in the regret.
This proves the goal of the second step.

\paragraph{Combining everything} 
Finally, note that the number of different values of $B$ \pref{alg:lossLB} uses is upper bounded by $\lceil\log_2(d^\star)\rceil=\lceil\log_2(D\mu^\star)\rceil$ since the optimal action $a^\star$ will never be eliminated when $B\geq \mu^\star$. Summing up the regret over these different values of $B$ arrives at the the final bound $\order(\min\{V_1,V_2\},\log(d^\star)\min\{W_1,W_2\})$.

\section{Extension to Contextual Linear Bandits}\label{sec: contextual}

In this section, we extend our results to the stochastic contextual setting where the action set at each round is drawn i.i.d. from a distribution $\dist$. 
While the arm elimination procedure is critical in solving our problem in the non-contextual case with a fixed action set, it is not clear (if possible at all) to directly generalize it to the contextual setting due to the dynamic nature of the action set.

Fortunately, a recent work by \citet{hanna2023contexts} proposes a reduction from contextual linear bandits to non-contextual linear bandits (both without delay).
At a high level, this reduction utilizes a subroutine of a non-contextual linear bandits algorithm by constructing a fixed action for each possible parameter $\theta$ of the contextual bandit instance. 
Importantly, the subroutine needs to be able to deal with an $\epsilon$-misspecified model, where the loss of each $a\in \calA$ is almost linear: $\mu_a=\inner{a,\theta}+\epsilon_a\in [0,1]$, with $\epsilon \geq \max_{a\in\calA}|\epsilon_a|$ indicating the misspecification level. 
It turns out that, a simple modification of our \pref{alg:lossLB} can address such misspecification --- it only requires incorporating the misspecification level $\epsilon$ into the criteria of arm elimination;
see \pref{alg:lossLBmis} and specifically its \pref{line:eliminate-mis} for details.

We then plugin this subroutine, denoted as $\Alg_{\nctx}$, into their reduction, as shown in \pref{alg:reduction}.
Specifically, the algorithm first constructs a $\frac{1}{T}$-cover $\Theta'$ of the parameter space $\Theta=\R_+^n\cap\mathbb{B}_2^n(1)$ with size $|\Theta'| = \mathcal{O}(T^n)$. 
It then proceeds in epochs with doubling length. 
At the start of epoch $m$, 
a new \emph{fixed} action set $\mathcal{X}_m = \{g^{(m)}(\theta) : \theta \in \Theta'\}$ is constructed, where $g^{(m)}(\theta)$ is the averaged optimal action over the previous $m-1$ epochs, assuming the model parameter being $\theta$.
Then, a new instance of $\Alg_{\nctx}$ with action set $\mathcal{X}_m$ and some 
misspecification level $\epsilon_m$ is initiated and run for the entire epoch.
At each round $t$ of this epoch, $\Alg_{\nctx}$ outputs an action $g^{(m)}(\theta_t) \in \mathcal{X}_m$, and the algorithm's final decision upon receiving the true action set $\calA_t$ is $a_t=\argmin_{a\in \calA_t}\inner{a,\theta_t}$.
Finally, at the end of this round, all newly observed losses are sent to $\Alg_{\nctx}$.

\begin{figure*}[t]
\centering

\includegraphics[width=0.33\textwidth]{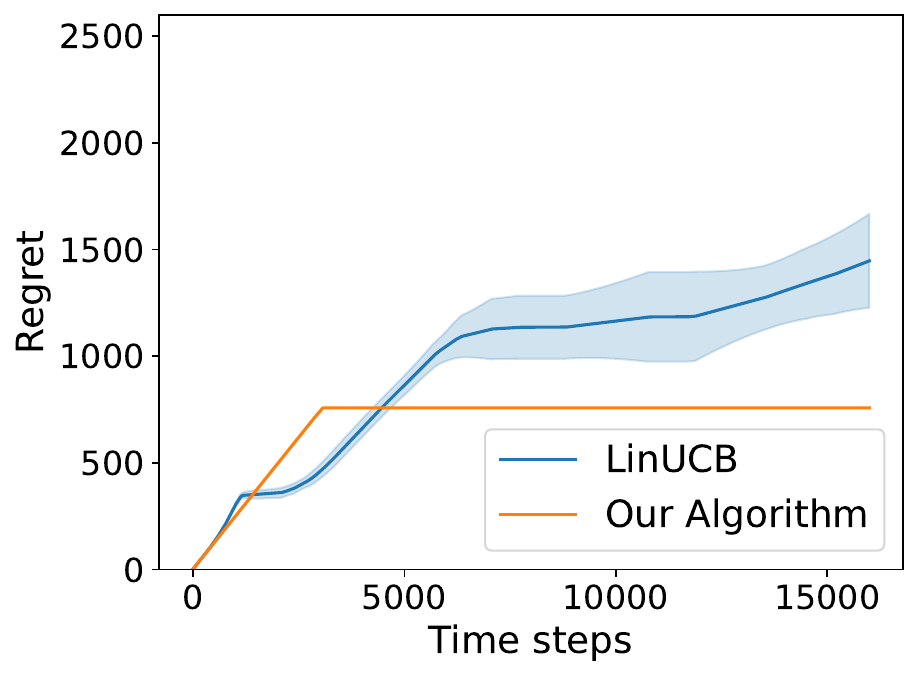}
\includegraphics[width=0.33\textwidth]{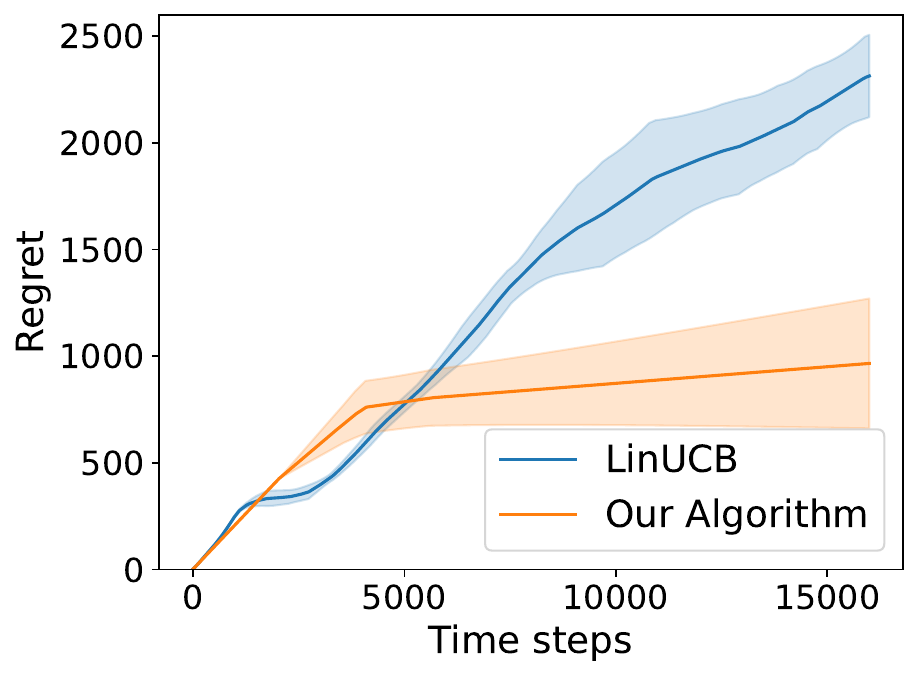}
\includegraphics[width=0.33\textwidth]{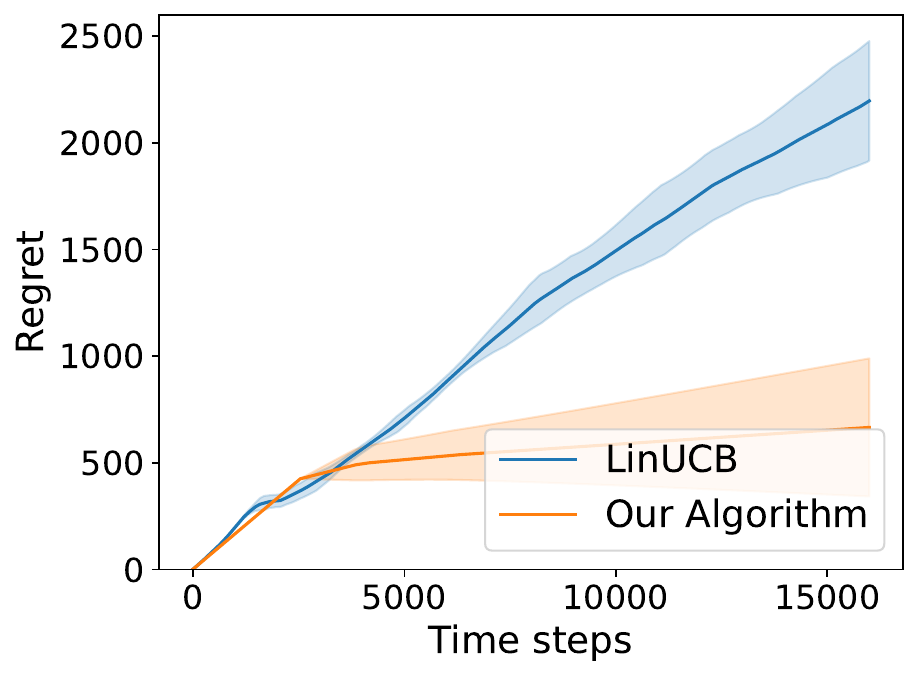}

\includegraphics[width=0.33\textwidth]{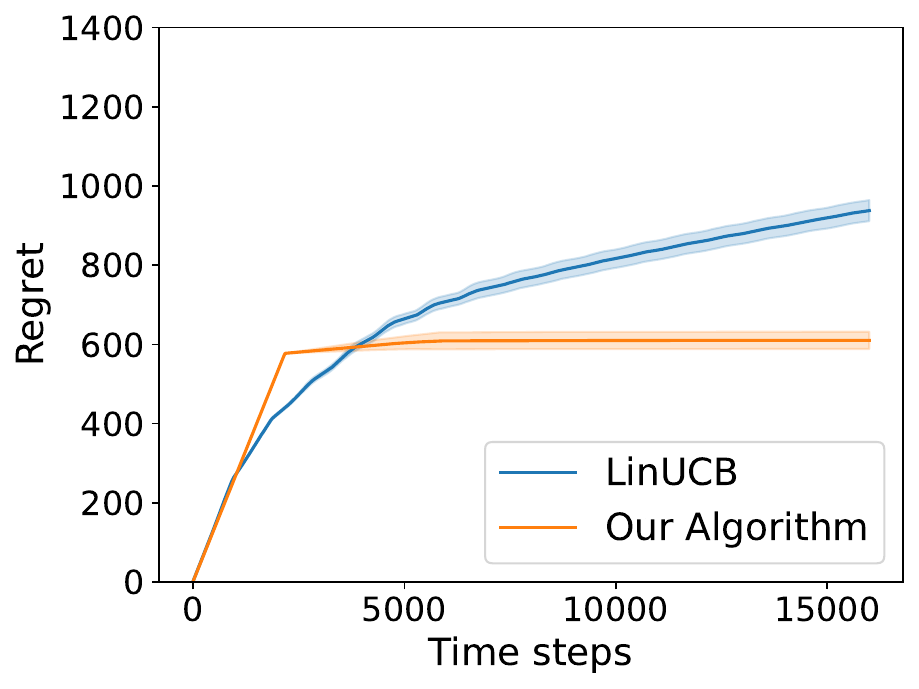}
\includegraphics[width=0.33\textwidth]{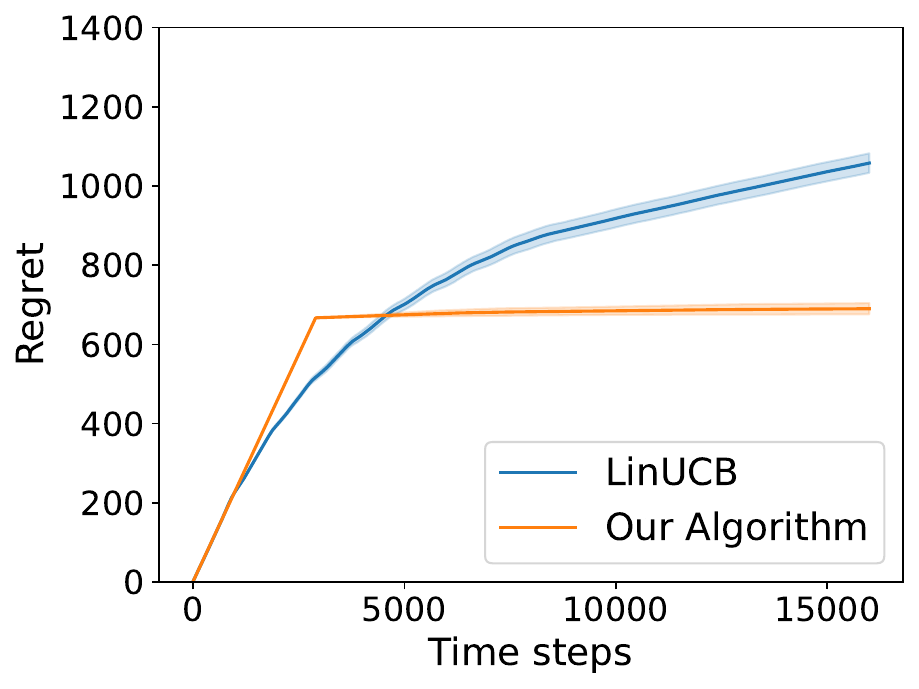}
\includegraphics[width=0.33\textwidth]{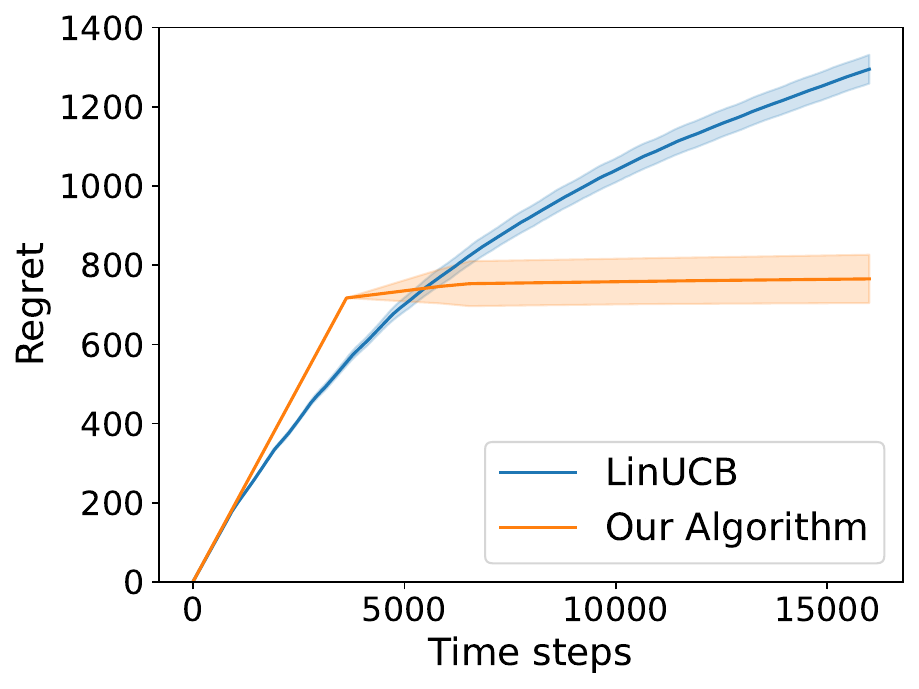}
\caption{Comparison of the empirical results of our algorithm and \texttt{LinUCB}. The top row is the delay-as-loss setting and the bottom row is the delay-as-reward setting. The left, middle, and right column correspond to $n=6,8,10$ respectively.}
\label{fig:synthetic_dataset}
\end{figure*}

\paragraph{Guarantees and Analysis}
Even though our algorithm is a direct application of the reduction of~\citet{hanna2023contexts}, it is a priori unclear whether it enjoys any favorable regret guarantee in the delay-as-loss setting.
By adopting and generalizing their analysis, we show that this is indeed the case.
Before introducing our results, we define the following quantities:
    \begin{align*}
        g(\theta) &\triangleq \E_{\calA\sim \dist}\left[\argmin_{a\in\calA}\inner{a,\theta}\right],\\
        \Delta_{\min}^{\nctx} &\triangleq\min_{\theta'\in \Theta', \inner{g(\theta'),\theta}\neq \inner{ g(\theta),\theta}}\E\left[\inner{g(\theta)-g(\theta'),\theta}\right],\\
        \Delta_{\max}^{\nctx}  &\triangleq\max_{\theta'\in \Theta'}\E\left[\inner{g(\theta)-g(\theta'),\theta}\right],\\
        \overline{d}^{\star} &\triangleq D\cdot \inner{g(\theta),\theta} = D\cdot \E_{\calA\sim \dist}\left[\min_{a\in \calA}\inner{a,\theta}\right],
    \end{align*}
    where $g(\theta)$ denotes the optimal action in expectation, $\Delta_{\min}^{\nctx}$ ($\Delta_{\max}^{\nctx}$) denotes the minimum (maximum) suboptimality gap for the reduced non-contextual linear bandit instance, and $\overline{d}^\star$ denotes the expected delay of the optimal action.

\begin{theorem}\label{thm:reduction}
    \pref{alg:reduction} with %$t^{(m)}=2^{m-1}$ and 
    $\delta = 1/T^2$ guarantees
    \begin{align*}
        &\Reg =\order\big(n\sqrt{T\log T}+\min\{V_1,V_2\} \\
        &\quad\quad\quad\quad +\log(\overline{d}^\star)\min\{W_1,W_2\}\big),
    \end{align*}
     where $V_1=\frac{n^3\log^2(T)\log(T/n)\log(\overline{d}^\star)}{\Delta_{\min}^{\nctx}}$, $V_2=n^{1.5}\sqrt{T\log(\overline{d}^\star)\log(T)}$, $W_1=\log T(n\overline{d}^\star\log(T/n)+D\Delta_{\max}^{\nctx})$, and $W_2=D\Delta_{\max}^{\nctx}\log T\log(T/n)$.
\end{theorem}
The proof is deferred to \pref{app: contextual}. 
The regret bound is in the same spirit as the one for the non-contextual case (\pref{thm:main-non-contextual}) and consists of a term for standard regret and a term for delay overhead.
The standard part unfortunately suffers higher dependence on the dimension $n$, while the delay overhead is in a similar problem-dependent form.
We remark that this is the first regret guarantee for contextual linear bandits with delay-as-payoff, resolving an open problem asked by \citep{schlisselberg2024delay}.
\section{Experiment}\label{sec: experiment}

In this section, we implement and evaluate our algorithm for both the delay-as-loss and delay-as-reward settings.
For simplicity, we only consider the non-contextual setting.
Since there are no existing algorithms for this problem (to the best of our knowledge), 
we consider a simple benchmark that applies the standard \texttt{LinUCB} algorithm only using the currently available observations (see \pref{alg:linUCB} in \pref{app: experiment} for details).
We point out that is simple approach to handling delayed feedback is indeed very common in the literature and in fact enjoys favorable guarantees at least for some problems~\citep{thune2019nonstochastic,van2023unified}.

\textbf{Experiment setup} The experiment setup is as follows. We set the dimension $n\in\{6,8,10\}$ and the size of the action set $|\calA|=50$. The model parameter $\theta$ is set to be $\frac{|\nu|}{\|\nu\|_2}$ where $\nu$ is drawn from the $n$-dimensional standard normal distribution and $|\nu|$ denotes the entry-wise absolute value of $\nu$ to make sure that $\theta\in \R_+^n\cap \mathbb{B}_2^n(1)$. Each action $a\in \calA$ is constructed by first sampling $a_i$ uniformly from $[0,1]$ for all $i\in [n]$ and then normalizing it to unit $\ell_2$-norm. When an action $a_t$ is chosen at round $t$, the payoff $u_t$ is defined as follows: with probability $1-\mu_{a_t}$, $u_t$ is drawn from $\calU[0,\mu_{a_t}]$, and with probability $\mu_{a_t}$, $u_t$ is drawn from $\calU[\mu_{a_t},1]$. It is a valid assignment since $\mu_{a_t}\in[0,1]$ and direct calculation shows that $\E[u_t]=(1-\mu_{a_t})\cdot \frac{\mu_{a_t}}{2}+\mu_{a_t}\cdot\frac{1+\mu_{a_t}}{2}=\mu_{a_t}$.
The number of iterations $T$ is $16000$ and the maximal possible delay $D$ is $1000$.
For simplicity, we also ignore the role of $B$ in our algorithms.

\textbf{Results} 
In \pref{fig:synthetic_dataset}, we plot the mean and the standard deviation of the regret over $8$ independent experiments with different random seeds, for each $n\in\{6,8,10\}$ (the columns) and also both delay-as-loss and delay-as-reward (the rows).
We observe that our algorithm consistently outperforms \texttt{LinUCB} in all setups. 
Also, in all runs, after about $9$ to $12$ epochs, our algorithm eliminates a significant number of bad actions, leading to almost constant regret after that point (and explaining the ``phase transition'' in the plots).

\section{Conclusion}
In this work, we initiate the study of the delay-as-payoff model for contextual linear bandits and develop provable algorithms that require novel ideas compared to standard linear bandits.
Interesting future directions include proving matching regret lower bounds and extending our results to general payoff-dependent delays~\citep{lancewicki2021stochastic} and other even more challenging settings, such as those with intermediate observations~\citep{esposito2023delayed} or
evolving observations~\citep{bar2024non}.

\bibliography{ref}

\begin{thebibliography}{25}
\providecommand{\natexlab}[1]{#1}
\providecommand{\url}[1]{\texttt{#1}}
\expandafter\ifx\csname urlstyle\endcsname\relax
  \providecommand{\doi}[1]{doi: #1}\else
  \providecommand{\doi}{doi: \begingroup \urlstyle{rm}\Url}\fi

\bibitem[Abbasi-Yadkori et~al.(2011)Abbasi-Yadkori, P{\'a}l, and
  Szepesv{\'a}ri]{abbasi2011improved}
Abbasi-Yadkori, Y., P{\'a}l, D., and Szepesv{\'a}ri, C.
\newblock Improved algorithms for linear stochastic bandits.
\newblock \emph{Advances in neural information processing systems}, 24, 2011.

\bibitem[Arya \& Yang(2020)Arya and Yang]{arya2020randomized}
Arya, S. and Yang, Y.
\newblock Randomized allocation with nonparametric estimation for contextual
  multi-armed bandits with delayed rewards.
\newblock \emph{Statistics \& Probability Letters}, 164:\penalty0 108818, 2020.

\bibitem[Bar-On \& Mansour(2024)Bar-On and Mansour]{bar2024non}
Bar-On, Y. and Mansour, Y.
\newblock Non-stochastic bandits with evolving observations.
\newblock \emph{arXiv preprint arXiv:2405.16843}, 2024.

\bibitem[Bhaskara et~al.(2023)Bhaskara, Mahabadi, and
  Vakilian]{bhaskara2023tight}
Bhaskara, A., Mahabadi, S., and Vakilian, A.
\newblock Tight bounds for volumetric spanners and applications.
\newblock In \emph{Thirty-seventh Conference on Neural Information Processing
  Systems}, 2023.
\newblock URL \url{https://openreview.net/forum?id=c4Xc0uTLXW}.

\bibitem[Blanchet et~al.(2024)Blanchet, Xu, and
  Zhou]{blanchet2024delay-contextual-linear}
Blanchet, J., Xu, R., and Zhou, Z.
\newblock Delay-adaptive learning in generalized linear contextual bandits.
\newblock \emph{Mathematics of Operations Research}, 49\penalty0 (1):\penalty0
  326--345, 2024.

\bibitem[Dud{\'i}k et~al.(2011)Dud{\'i}k, Hsu, Kale, Karampatziakis, Langford,
  Reyzin, and Zhang]{Dudk2011EfficientOL}
Dud{\'i}k, M., Hsu, D.~J., Kale, S., Karampatziakis, N., Langford, J., Reyzin,
  L., and Zhang, T.
\newblock Efficient optimal learning for contextual bandits.
\newblock In \emph{Conference on Uncertainty in Artificial Intelligence}, 2011.

\bibitem[Esposito et~al.(2023)Esposito, Masoudian, Qiu, Van Der~Hoeven,
  Cesa-Bianchi, and Seldin]{esposito2023delayed}
Esposito, E., Masoudian, S., Qiu, H., Van Der~Hoeven, D., Cesa-Bianchi, N., and
  Seldin, Y.
\newblock Delayed bandits: when do intermediate observations help?
\newblock \emph{International Conference on Machine Learning}, 2023.

\bibitem[Gael et~al.(2020)Gael, Vernade, Carpentier, and
  Valko]{gael2020stochastic}
Gael, M.~A., Vernade, C., Carpentier, A., and Valko, M.
\newblock Stochastic bandits with arm-dependent delays.
\newblock In \emph{International Conference on Machine Learning}, pp.\
  3348--3356. PMLR, 2020.

\bibitem[Gyorgy \& Joulani(2021)Gyorgy and Joulani]{gyorgy2021adapting}
Gyorgy, A. and Joulani, P.
\newblock Adapting to delays and data in adversarial multi-armed bandits.
\newblock In \emph{International Conference on Machine Learning}, pp.\
  3988--3997. PMLR, 2021.

\bibitem[Hanna et~al.(2023)Hanna, Yang, and Fragouli]{hanna2023contexts}
Hanna, O.~A., Yang, L., and Fragouli, C.
\newblock Contexts can be cheap: Solving stochastic contextual bandits with
  linear bandit algorithms.
\newblock In \emph{The Thirty Sixth Annual Conference on Learning Theory}, pp.\
   1791--1821. PMLR, 2023.

\bibitem[Hazan \& Karnin(2016)Hazan and Karnin]{hazan2016volumetric}
Hazan, E. and Karnin, Z.
\newblock Volumetric spanners: an efficient exploration basis for learning.
\newblock \emph{Journal of Machine Learning Research}, 2016.

\bibitem[Joulani et~al.(2013)Joulani, Gyorgy, and
  Szepesv{\'a}ri]{joulani2013online}
Joulani, P., Gyorgy, A., and Szepesv{\'a}ri, C.
\newblock Online learning under delayed feedback.
\newblock In \emph{International Conference on Machine Learning}, pp.\
  1453--1461. PMLR, 2013.

\bibitem[Lancewicki et~al.(2021)Lancewicki, Segal, Koren, and
  Mansour]{lancewicki2021stochastic}
Lancewicki, T., Segal, S., Koren, T., and Mansour, Y.
\newblock Stochastic multi-armed bandits with unrestricted delay distributions.
\newblock In \emph{International Conference on Machine Learning}, pp.\
  5969--5978. PMLR, 2021.

\bibitem[Li et~al.(2010)Li, Chu, Langford, and Schapire]{li2010contextual}
Li, L., Chu, W., Langford, J., and Schapire, R.~E.
\newblock A contextual-bandit approach to personalized news article
  recommendation.
\newblock In \emph{Proceedings of the 19th international conference on World
  wide web}, pp.\  661--670, 2010.

\bibitem[Mandel et~al.(2015)Mandel, Liu, Brunskill, and
  Popovi{\'c}]{mandel2015queue}
Mandel, T., Liu, Y.-E., Brunskill, E., and Popovi{\'c}, Z.
\newblock The queue method: Handling delay, heuristics, prior data, and
  evaluation in bandits.
\newblock In \emph{Proceedings of the AAAI Conference on Artificial
  Intelligence}, volume~29, 2015.

\bibitem[Pike-Burke et~al.(2018)Pike-Burke, Agrawal, Szepesvari, and
  Grunewalder]{pike2018bandits}
Pike-Burke, C., Agrawal, S., Szepesvari, C., and Grunewalder, S.
\newblock Bandits with delayed, aggregated anonymous feedback.
\newblock In \emph{International Conference on Machine Learning}, pp.\
  4105--4113. PMLR, 2018.

\bibitem[Schlisselberg et~al.(2024)Schlisselberg, Cohen, Lancewicki, and
  Mansour]{schlisselberg2024delay}
Schlisselberg, O., Cohen, I., Lancewicki, T., and Mansour, Y.
\newblock Delay as payoff in mab.
\newblock \emph{arXiv preprint arXiv:2408.15158}, 2024.

\bibitem[Tang et~al.(2024)Tang, Wang, and Zheng]{tang2024stochastic}
Tang, Y., Wang, Y., and Zheng, Z.
\newblock Stochastic multi-armed bandits with strongly reward-dependent delays.
\newblock In \emph{International Conference on Artificial Intelligence and
  Statistics}, pp.\  3043--3051. PMLR, 2024.

\bibitem[Thune et~al.(2019)Thune, Cesa-Bianchi, and
  Seldin]{thune2019nonstochastic}
Thune, T.~S., Cesa-Bianchi, N., and Seldin, Y.
\newblock Nonstochastic multiarmed bandits with unrestricted delays.
\newblock \emph{Advances in Neural Information Processing Systems}, 32, 2019.

\bibitem[Van Der~Hoeven \& Cesa-Bianchi(2022)Van Der~Hoeven and
  Cesa-Bianchi]{van2022nonstochastic}
Van Der~Hoeven, D. and Cesa-Bianchi, N.
\newblock Nonstochastic bandits and experts with arm-dependent delays.
\newblock In \emph{International Conference on Artificial Intelligence and
  Statistics}. PMLR, 2022.

\bibitem[van~der Hoeven et~al.(2023)van~der Hoeven, Zierahn, Lancewicki,
  Rosenberg, and Cesa-Bianchi]{van2023unified}
van~der Hoeven, D., Zierahn, L., Lancewicki, T., Rosenberg, A., and
  Cesa-Bianchi, N.
\newblock A unified analysis of nonstochastic delayed feedback for
  combinatorial semi-bandits, linear bandits, and mdps.
\newblock In \emph{The Thirty Sixth Annual Conference on Learning Theory}, pp.\
   1285--1321. PMLR, 2023.

\bibitem[Vernade et~al.(2017)Vernade, Capp{\'e}, and
  Perchet]{vernade:hal-01545667}
Vernade, C., Capp{\'e}, O., and Perchet, V.
\newblock {Stochastic Bandit Models for Delayed Conversions}.
\newblock In \emph{{Conference on Uncertainty in Artificial Intelligence}},
  2017.

\bibitem[Vernade et~al.(2020)Vernade, Carpentier, Lattimore, Zappella, Ermis,
  and Brueckner]{vernade2020linear}
Vernade, C., Carpentier, A., Lattimore, T., Zappella, G., Ermis, B., and
  Brueckner, M.
\newblock Linear bandits with stochastic delayed feedback.
\newblock In \emph{International Conference on Machine Learning}, pp.\
  9712--9721. PMLR, 2020.

\bibitem[Zhou et~al.(2019)Zhou, Xu, and Blanchet]{zhou2019learning}
Zhou, Z., Xu, R., and Blanchet, J.
\newblock Learning in generalized linear contextual bandits with stochastic
  delays.
\newblock \emph{Advances in Neural Information Processing Systems}, 32, 2019.

\bibitem[Zimmert \& Seldin(2020)Zimmert and Seldin]{zimmert2020optimal}
Zimmert, J. and Seldin, Y.
\newblock An optimal algorithm for adversarial bandits with arbitrary delays.
\newblock In \emph{International Conference on Artificial Intelligence and
  Statistics}, pp.\  3285--3294. PMLR, 2020.

\end{thebibliography}
\bibliographystyle{icml2025}

\newpage
\appendix
\onecolumn
\section{Omitted Details in \pref{sec: linear}}\label{app:loss}

\setcounter{AlgoLine}{0}
\begin{algorithm}
\caption{Phased Elimination via Volumetric Spanner for Linear Bandits with Delay-as-Loss with misspecification}\label{alg:lossLBmis}

\nl Input: maximum possible delay $D$, action set $\calA$, $\beta>0$, a misspecification level $\epsilon$. 

\nl Initialization: optimal loss guess $B=1/D$.

\nl Initialization: active action set $\calA_1=\calA$. \label{line: restart-mis}

 \For{$m=1,2,\dots,$}{
    \nl Find $\calS_m=\{a_{m,1},\dots,a_{m,|\calS_m|}\}$, a volumetric spanner of $\calA_m$ with $|\calS_m|= 3n$. \label{line:volume-mis}
    
    \nl Pick each $a\in \calS_m$ $2^m$ times in a round-robin way. \label{line:round-robin-mis}

    \nl Let $\calI_m$ contain all the rounds in this epoch.
    
    \nl For each $a\in \calS_m$, calculate the following quantities: \label{line:spanner-ucb-lcb-mis}
    {\small
    \begin{align}
        &\hat{\mu}_{m}^+(a)=\frac{1}{2^m}\Big(\sum_{\tau\in \obs_m(a)}u_{\tau} + \sum_{\tau\in \unobs_m(a)}1\Big), \label{eqn:mean-up-mis}\\
        &\hat{\mu}_{m}^-(a)=\frac{1}{2^m}\sum_{\tau\in \obs_m(a)}u_{\tau}, \label{eqn:mean-low-mis}\\
        &\hat{\mu}_{m,1}^{+}(a)=\hat{\mu}^{+}_{m}(a)+\frac{\beta}{2^{m/2}}\|a\|_2, \label{eqn:loss-ucb-linear-1-mis}\\
        &\hat{\mu}_{m,1}^{-}(a)=\hat{\mu}^{-}_{m}(a)-\frac{\beta}{2^{m/2}}\|a\|_2,\label{eqn:loss-lcb-linear-1-mis}\\
        &\hat{\mu}_m^{F}(a)=\frac{1}{\unbiasSize_m(a)}\sum_{\tau\in \unbias_m(a)}u_{\tau}, \label{eqn:mean_unbiased-mis}\\
        &\hat{\mu}_{m,2}^{+}(a)=\hat{\mu}_m^F(a)+\frac{\beta}{\sqrt{\unbiasSize_m(a)}}\|a\|_2, \label{eqn:loss-ucb-linear-2-mis}\\
        &\hat{\mu}_{m,2}^{-}(a)=\hat{\mu}_m^F(a)-\frac{\beta}{\sqrt{\unbiasSize_m(a)}}\|a\|_2, \label{eqn:loss-lcb-linear-2-mis}
    \end{align}
    }
    where $\unbiasSize_m(a) = |\unbias_m(a)|$, $\unbias_m(a) = \{\tau\in \calI_m: \tau+D\in\calI_m, a_{\tau}=a\}$, $\obs_m(a) = \{\tau\in \calI_m: \tau+d_{\tau}\in\calI_m, a_{\tau}=a\}$, and
    $\unobs_m(a)= \{\tau\in \calI_m: a_{\tau}=a\}\setminus\obs_m(a)$.

    \For{each $a\in \calA_m$}{
        \nl \label{line: decompose-mis}
        Decompose $a$ as $a=\sum_{i=1}^{|S_m|}\lambda_{m,i}^{(a)}a_{m,i}$ with $\|\lambda_{m}^{(a)}\|_2\leq 1$ and calculate 
        {\small
        \begin{align}
            &\UCB_{m}(a)=\sum_{i=1}^{|\calS_m|}\lambda_{m,i}^{(a)}\cdot\hat{\mu}_{m,2}^{\sgn(\lambda_{m,i}^{(a)})}(a_{m,i}), \label{eqn:loss-ucb-f-all-action-mis} \\
            &\LCB_m(a) = \max_{j\in \{1,2\}}\{\LCB_{m,j}(a)\} \;\;\text{where} \nonumber  \\
            & \LCB_{m,j}(a)=\sum_{i=1}^{|\calS_m|}\lambda_{m,i}^{(a)}\cdot\hat{\mu}_{m,j}^{\sgn(-\lambda_{m,i}^{(a)})}(a_{m,i}),\label{eqn:loss-lcb-all-action-mis}
        \end{align}
        }
    }
    
    \nl Set $\calA_{m+1} = \calA_m$.
    
    \For{$a\in \calA_m$}{
        \nl \label{line:eliminate-mis}  
        \If{$\exists a'\in \calA_m$, s.t. $\LCB_m(a) \geq \min\{\UCB_m(a'),B\} + 4\sqrt{3n}\epsilon$}
        {
          Eliminate $a$ from $\calA_{m+1}$.
        }
    }
    \nl \If{$\calA_{m+1}=\emptyset$}{
        Set $B\leftarrow 2B$ and go to \pref{line: restart-mis}.
    }
}
\end{algorithm}

In this section, we provide the detailed proof for \pref{thm:main-non-contextual}. Specifically, as mentioned in \pref{sec: contextual}, we prove the guarantee of a modified algorithm (\pref{alg:lossLBmis}) for the more general $\epsilon$-misspecified linear bandits. 

Recall that in misspecified linear bandits, $\mu_a = \inner{a,\theta}+\epsilon_a\in[0,1]$ with $|\epsilon_a|\leq\epsilon$ for all $a\in\calA$. Due to this difference, we clarify on the definitions of $\Delta_a$, $a^\star$, $\mu^\star$, $\Delta_{\min}$, $\Delta_{\max}$, and $d^\star$ in misspecified linear bandits as follows. We still define $\Delta_a = \inner{a^\star-a,\theta}$ as the suboptimality gap of action $a$, where $a^\star \in \argmin_{a\in\calA}\inner{a, \theta}$, but $\mu^\star \triangleq \min_{a\in\calA}\mu_a$ as the loss of the optimal action. Note that due to the misspecification, $\mu^\star$ may not necessarily be $\mu_{a^\star}$. Define $\Delta_{\min} = \min_{a\in \calA, \Delta_a>0}\Delta_a$ and $\Delta_{\max} = \max_{a\in \calA}\Delta_a$ to be the minimum non-zero, and maximum sub-optimality gap. The delay at round $t$ is still defined as $d_t=D\cdot u_t$ and $d^\star = D\cdot \mu^\star$ is the expected delay of the optimal action.

As for the algorithm, \pref{alg:lossLBmis} differs from \pref{alg:lossLB} only in \pref{line:eliminate-mis} where we add one misspecification term $4\sqrt{3n}\epsilon$ in the criteria of eliminating an action. 

The following theorem shows the guarantee of our algorithm in the misspecified linear bandits.

\begin{theorem}\label{thm:lossLBmis}
    \pref{alg:lossLBmis} with $\beta = \sqrt{2\log(KT^3)}$ guarantees that
    \begin{align*}
        \Reg &\leq \order\left(\min\left\{\frac{n^2\log(KT)\log(T/n)\log(d^\star)}{\Delta_{\min}},n\sqrt{T\log(d^\star)\log(KT)}\right\}+\epsilon\sqrt{n}T\right) \\
        &\qquad + \log(d^\star)\cdot \order\left( \min\left\{nd^\star\log (T/n)+D\Delta_{\max},D\Delta_{\max}\log (T/n)\right\}\right).
    \end{align*}
\end{theorem}

To prove \pref{thm:lossLBmis}, recall the following quantities
\begin{align}
    \wh{\mu}_{m}(a) &= \frac{1}{2^m}\sum_{\tau\in\obs_m(a)\cup\unobs_m(a)}u_{\tau},~~~\forall a\in \calS_m,\label{eqn:loss-all-mean-app}\\
    \hat{\mu}_{m,1}(a)&=\sum_{i=1}^{|\calS_m|}\lambda_{m,i}^{(a)}\cdot\hat{\mu}_{m}(a_{m,i}),~~~\forall a\in \calA_m,\\
    \hat{\mu}_{m,2}(a)&=\sum_{i=1}^{|\calS_m|}\lambda_{m,i}^{(a)}\cdot\hat{\mu}_{m}^{F}(a_{m,i}),~~~\forall a\in \calA_m.
\end{align}
We then define the following event and show that the event holds with high probability.

\begin{event}\label{event:misLoss}
    For all action $a\in \calA_m$, $m\in[T]$,
    \begin{align}
        \left|\inner{a,\theta}-\hat{\mu}_{m,1}(a)\right|&\leq \sqrt{|\calS_m|}\epsilon + \beta\sum_{i=1}^{|\calS_m|}\left|\lambda_{m,i}^{(a)}\right|\sqrt{\frac{1}{2^{m}}}, \label{eqn:concentr-1}\\
        \left|\inner{a,\theta}-\hat{\mu}_{m,2}(a)\right|&\leq \sqrt{|\calS_m|}\epsilon +\beta \sum_{i=1}^{|\calS_m|}\left|\lambda_{m,i}^{(a)}\right|\sqrt{\frac{1}{\unbiasSize_m(a_{m,i})}}, \label{eqn:concentr-2}\\
        |\unobs_m(a)| &\leq \frac{2D\mu_a}{|\calS_m|}+16\log KT+2,\label{eqn:concentr-3}
    \end{align}
    where $\beta = \sqrt{2\log KT^3}$.
\end{event}
\begin{lemma}\label{lem:high-prob-event}
    \pref{alg:lossLBmis} guarantees that \pref{event:misLoss} holds with probability at least $1-\frac{2}{T^2}$.
\end{lemma}
\begin{proof}
    Fix an action $a\in \calS_m$ in epoch $m\in[T]$. According to standard Azuma's inequality, we know that with probability at least $1-\delta$,
    \begin{align*}
        \left|\mu_a-\hat{\mu}_{m,1}(a)\right|&\leq \sqrt{\frac{2\log(2/\delta)}{2^m}}\|a\|_2,\\
        \left|\mu_a-\hat{\mu}_{m,2}(a)\right|&\leq \sqrt{\frac{2\log(2/\delta)}{\unbiasSize_m(a)}}\|a\|_2.
    \end{align*}
    Taking union bound over all possible $a\in \calA$ and all $m\in[T]$, we know that with probability at least $1-\delta$, for all $a\in \calS_m$ and all $m\in [T]$,
    \begin{align*}
        \left|\mu_a-\hat{\mu}_{m,1}(a)\right|&\leq \sqrt{\frac{2\log(2TK/\delta)}{n_t(a)}}\|a\|_2,\\
        \left|\mu_a-\hat{\mu}_{m,2}(a)\right|&\leq \sqrt{\frac{2\log(2TK/\delta)}{\unbiasSize_m(a)}}\|a\|_2.
    \end{align*}
    Then, given that the above equation holds, for $a\in \calA_m$, due to the property of volumetric spanners, we have $\mu_a = \inner{a,\theta^\star}+\epsilon_a =  \sum_{i=1}^{|\calS_m|}\lambda_{m,i}^{(a)}\inner{a_{m,i},\theta^\star}+ \epsilon_a$. Therefore, we can obtain that
    \begin{align*}
        \left|\inner{a,\theta}-\hat{\mu}_{m,1}(a)\right|
        &\leq \left|\sum_{i=1}^{|\calS_m|}\lambda_{m,i}^{(a)}(\inner{a_{m,i},\theta^\star}-\mu_{a_{m,i}})\right| + \sum_{i=1}^{|\calS_m|}\left|\lambda_{m,i}^{(a)}\right|\cdot\left|\mu_{a_{m,i}}-\hat{\mu}_{m}(a_{m,i})\right| \\
        &\leq \sum_{i=1}^{|\calS_m|}\left|\lambda_{m,i}^{(a)}\right|\left(\epsilon_{a_{m,i}}+\sqrt{\frac{2\log(2TK/\delta)}{2^{m}}}\right) \\
        &\leq \sqrt{|\calS_m|}\epsilon + \sum_{i=1}^{|\calS_m|}\left|\lambda_{m,i}^{(a)}\right|\sqrt{\frac{2\log(2TK/\delta)}{2^{m}}},
    \end{align*}
    where the last inequality uses $\|\lambda_{m}^{(a)}\|_1\leq \sqrt{|\calS_m|}\cdot \|\lambda_{m}^{(a)}\|_2\leq \sqrt{|\calS_m|}$. A similar analysis proves \pref{eqn:concentr-2}. \pref{eqn:concentr-3} holds with probability at least $1-\frac{1}{T^2}$ according to Lemma 4.1 of \citep{schlisselberg2024delay}. Picking $\delta = \frac{1}{T^2}$ finishes the proof.
\end{proof}

The next lemma shows that if $B\geq \mu^\star$, then \pref{alg:lossLBmis} will not reach an empty active set.
\begin{lemma}\label{lem:end-of-B}
    Suppose that \pref{event:misLoss} holds. If $B\geq \mu^\star$, then $a^\star\in \calA_m$ for all $m$.
\end{lemma}
\begin{proof}
    Since \pref{event:misLoss} holds, we have, we know that for all $a\in\calA_m$, $\LCB_m(a)\leq \inner{a,\theta} + \sqrt{|\calS_m|}\epsilon $ and $\UCB_m(a)\geq \inner{a,\theta} - \sqrt{|\calS_m|}\epsilon$. If $B\geq \mu^\star$, then we have $a^\star$ never eliminated since for any $a\in \calA_m$
    \begin{align*}
        \LCB_{m}(a^\star) &\leq \inner{a^\star,\theta} + \epsilon\sqrt{|\calS_m|} \leq \mu^\star + \epsilon + \epsilon\sqrt{|\calS_m|} \leq \mu^\star + 2\epsilon\sqrt{|\calS_m|},\\
        \LCB_{m}(a^\star) &\leq \inner{a^\star,\theta} + \epsilon\sqrt{|\calS_m|} \leq \inner{a,\theta} + 2\epsilon\sqrt{|\calS_m|} \leq \UCB_m(a) + 4\epsilon\sqrt{|\calS_m|}.
    \end{align*}
    Therefore, $a^\star$ never satisfy the elimination condition.
\end{proof}

The following lemma shows that the regret within epoch $m$ can be well-controlled.

\begin{lemma}\label{lem:delta_1_loss_miss}
    Suppose that \pref{event:misLoss} holds. \pref{alg:lossLBmis} guarantees that if $a\in\calA$ is not eliminated at the end of epoch $m$ (meaning that $a\in \calA_{m+1}$), then 
    \begin{align*}
        2^m\cdot \Delta_a\leq 2^m\cdot 24\sqrt{n}\epsilon+\frac{256n\beta^2}{\Delta_a} + \frac{2D\Delta_a}{|\calS_m|}.
    \end{align*}
\end{lemma}
\begin{proof}
    For notational convenience, define $\rad_{m,a}^{N} = \frac{\beta}{\sqrt{2^m}}\|a\|_2$ and $\rad_{m,a}^{F} = \frac{\beta}{\sqrt{\unbiasSize_m(a)}}\|a\|_2$ for all $a\in \calS_m$. In addition, we also define $\rad_{m,a}^{N}$ and $\rad_{m,a}^{F}$ for $a\notin \calS_m$ as follows:
    \begin{align*}
        \rad_{m,a}^{N} &= \sum_{i=1}^{|\calS_m|}|\lambda_{m,i}^{(a)}|\cdot \rad_{m,a_{m,i}}^{N}, \\
        \rad_{m,a}^{F} &= \sum_{i=1}^{|\calS_m|}|\lambda_{m,i}^{(a)}|\cdot \rad_{m,a_{m,i}}^{F}.
    \end{align*}
    Since \pref{event:misLoss} holds, we know that for all $a\in\calA_m$, $\LCB_m(a)\leq \inner{a,\theta} + \sqrt{|\calS_m|}\epsilon$, $\UCB_m(a)\geq \inner{a,\theta} - \sqrt{|\calS_m|}\epsilon$. Moreover, as $\LCB_m(a)=\max\{\LCB_{m,1}(a),\LCB_{m,2}(a)\}$, we know that for all $a\in \calA_m$
    \begin{align*}
        \LCB_{m,1}(a) + 2\rad_{m,a}^{N} + 2\epsilon\sqrt{|\calS_m|}\geq  \hat{\mu}_{m,1}(a) + \rad_{m,a}^{N} + 2\epsilon\sqrt{|\calS_m|}\geq \inner{a,\theta},\\
        \LCB_{m,2}(a) + 2\rad_{m,a}^{F} + 2\epsilon\sqrt{|\calS_m|}\geq  \hat{\mu}_{m,2}(a) + \rad_{m,a}^{F} + 2\epsilon\sqrt{|\calS_m|}\geq \inner{a,\theta},\\
        \UCB_{m}(a) - 2\rad_{m,a}^{F} - 2\epsilon\sqrt{|\calS_m|} = \hat{\mu}_{m,2}(a) - \rad_{m,a}^{F}-2\epsilon\sqrt{|\calS_m|}\leq \inner{a,\theta}.        
    \end{align*}
    If $B\geq \mu^\star$, then $a^\star\in \calA_m$ according to \pref{lem:end-of-B}.
    Moreover, if $a$ is not eliminated in epoch $m$, we have $\LCB(a)\leq \min\{\UCB_m(a^\star),B\}+4\sqrt{|S_m|}\epsilon$, meaning that
    \begin{align*}
        &\inner{a,\theta} - 2\rad_{m,a}^{F} - 2\epsilon\sqrt{|\calS_m|} \\
        &\leq \wh{\mu}_{m,2}(a) - \rad_{m,a}^{F} \\
        &\leq \LCB_m(a) \\
        &\leq \min\{\UCB_m(a^\star),B\}+4\sqrt{|S_m|}\epsilon \\
        &\leq \UCB_m(a^\star) + 4\sqrt{|S_m|}\epsilon \\
        &= \wh{\mu}_{m,2}(a^\star) + \rad_{m,a^\star}^{F}+ 4\sqrt{|S_m|}\epsilon \\
        &\leq \inner{a^\star,\theta} + 2\rad_{m,a^\star}^{F} + 6\sqrt{|S_m|}\epsilon.
    \end{align*}
    Since $\rad_{m,a}^F = \sum_{i=1}^{|\calS_m|}|\lambda_{m,i}^{(a)}|\cdot \rad_{m,a_{m,i}}^{F}$ with $\|\lambda_{m}^{(a)}\|_2\leq 1$, we have that $\|\lambda_{m}^{(a)}\|_1\leq \sqrt{|\calS_m|}$ and
    \begin{align*}
        &\Delta_a\leq 4\sqrt{|\calS_m|}\left(\max_{a\in \calS_m}\rad_{m,a}^{F}+2\epsilon\right)= 4\sqrt{3n}\max_{a\in \calS_m}\rad_{m,a}^{F}+8\sqrt{3n}\epsilon \leq \frac{8\sqrt{n}\beta}{\min_{a'\in \calS_m}\sqrt{\unbiasSize_m(a')}}+16\sqrt{n}\epsilon.
    \end{align*}
    
    If $B\leq \mu^\star$, then we have
    \begin{align*}
        \inner{a^\star,\theta}+\epsilon\geq \mu^\star\geq B \geq \LCB_{m}(a) - 4\sqrt{|\calS_m|}\epsilon \geq \inner{a,\theta} - 2\rad_{m,a}^{F} - 5\sqrt{|\calS_m|}\epsilon,
    \end{align*}
    where the second inequality is because $a$ is not eliminated in epoch $m$. Therefore, we always have
    \begin{align*}
        \Delta_a &\leq 2\rad_{m,a}^{F} + 6\sqrt{|\calS_m|}\epsilon \leq \frac{8\sqrt{n}\beta}{\min_{a'\in \calS_m}\sqrt{\unbiasSize_m(a')}} + 12\sqrt{n}\epsilon.
    \end{align*}
    In addition, we know that for all $a\in \calS_m$,
    \begin{align*}
        2^m \leq \unbiasSize_m(a) + \frac{D}{|\calS_m|} + 1 \leq \unbiasSize_m(a) + \frac{2D}{|\calS_m|}.
    \end{align*}
    Therefore, if $12\sqrt{n}\epsilon\geq \frac{\Delta_a}{2}$, then we have
    \begin{align*}
        2^m\Delta_a\leq 2^m\cdot 24\sqrt{n}\epsilon;
    \end{align*}
    otherwise, we have $\Delta_a \leq \frac{8\sqrt{n}\beta}{\min_{a\in \calS_m}\sqrt{\unbiasSize_m(a)}} + 12\sqrt{n}\epsilon \leq \frac{8\sqrt{n}\beta}{\min_{a\in \calS_m}\sqrt{\unbiasSize_m(a)}}  + \frac{\Delta_a}{2}$ and
    \begin{align*}
        \Delta_a \leq \frac{16\sqrt{n}\beta}{\min_{a'\in \calS_m}\sqrt{\unbiasSize_m(a')}},
    \end{align*}
    and we can obtain that
    \begin{align*}
        \min_{a'\in \calS_m}{\unbiasSize_m(a')}\cdot \Delta_a\leq \frac{256d\beta^2}{\Delta_a}.
    \end{align*}
    Combining the above two cases, we know that for all $a\in\calA_m$, $$2^m\cdot \Delta_a\leq 2^m\cdot 24\sqrt{n}\epsilon+ \min_{a'\in \calS_m}\unbiasSize_m(a')\cdot \Delta_a + \frac{2D\Delta_a}{|\calS_m|} \leq  2^m\cdot 24\sqrt{n}\epsilon+\frac{256n\beta^2}{\Delta_a} + \frac{2D\Delta_a}{|\calS_m|}.$$
\end{proof}

In fact, the bound above can be obtained by only using $\LCB_{m,1}$. Next, we provide yet-another regret bound within epoch $m$, which utilizes $\LCB_{m,2}$.

\begin{lemma}\label{lem:epoch_B_with_mis}
    \pref{alg:lossLBmis} guarantees that under \pref{event:misLoss}, if action $a$ is not eliminated at the end of epoch $m$ (meaning that $a\in \calA_{m+1}$), then
    \begin{align*}
    \inner{a,\theta}\leq B +\rad_{m,a}^{N}+ \sum_{i=1}^{|\calS_m|}|\lambda_{m,i}^{(a)}|\cdot\left(\frac{2D\mu_{a_{m,i}}}{2^m|\calS_m|}+\frac{16\log T +2}{2^m}\right) + 8\sqrt{|\calS_m|}\epsilon.
\end{align*}
\end{lemma}
\begin{proof}
For all $a\in \calS_m$, since $u_{t}\in[0,1]$, we know that
\begin{align}
    \hat{\mu}_{m}^+(a) &=\frac{1}{2^m}\left(\sum_{\tau\in \obs_m(a)}u_{\tau} + \sum_{\tau\in \unobs_m(a)}1\right) \leq \hat{\mu}_{m,a} + \frac{|\unobs_m(a)|}{2^m}, \label{eqn:pos-bias}\\
    \hat{\mu}_{m}^-(a) &=\frac{1}{2^m}\left(\sum_{\tau\in \obs_m(a)}u_{\tau} \right) \geq \hat{\mu}_{m,a} - \frac{|\unobs_m(a)|}{2^m} \label{eqn:neg-bias}.
\end{align}
Then, under \pref{event:misLoss}, we know that for all $a\in \calA_m$,
\begin{align*}
    \inner{a,\theta} &= \sum_{i=1}^{|\calS_m|}\lambda_{m,i}^{(a)}\inner{a_{m,i},\theta^\star}\\
    &= \sum_{i=1}^{|\calS_m|}\lambda_{m,i}^{(a)}(\mu_{a_{m,i}}-\epsilon_{a_{m,i}}) \tag{since $\mu_a = \inner{a,\theta^\star}+\epsilon_a$}\\
    &\leq \sum_{i=1}^{|\calS_m|}\lambda_{m,i}^{(a)}\cdot \mu_{a_{m,i}} + \sqrt{|\calS_m|}\epsilon \tag{since $\|\lambda_{m}^{(a)}\|_1\leq \sqrt{|\calS_m|}$} \\
    &\leq \sum_{i=1}^{|\calS_m|}\lambda_{m,i}^{(a)}\cdot \hat{\mu}_{m}(a_{m,i}) + \rad_{m,a}^{N} + 3\sqrt{|\calS_m|}\epsilon \tag{since \pref{event:misLoss} holds}\\
    &\leq \sum_{i=1}^{|\calS_m|}\lambda_{m,i}^{(a)}\cdot\hat{\mu}_{m}^{sgn(-\lambda_{m,i}^{(a)})}(a_{m,i}) + \rad_{m,a}^{N}+\sum_{i=1}^{|\calS_m|}|\lambda_{m,i}^{(a)}|\cdot \frac{|\unobs_m(a_{m,i})|}{2^m} + 3\sqrt{|\calS_m|}\epsilon \tag{using \pref{eqn:pos-bias} and \pref{eqn:neg-bias}}\\
    &= \LCB_{m,1}(a) + \rad_{m,a}^{N}+\sum_{i=1}^{|\calS_m|}|\lambda_{m,i}^{(a)}|\cdot \frac{|\unobs_m(a_{m,i})|}{2^m} + 3\sqrt{|\calS_m|}\epsilon\\
    &\leq \LCB_{m,1}(a) +\rad_{m,a}^{N}+ \sum_{i=1}^{|\calS_m|}|\lambda_{m,i}^{(a)}|\cdot\left(\frac{2D\mu_{a_{m,i}}}{2^m|\calS_m|}+\frac{16\log KT +2}{2^m}\right) + 3\sqrt{|\calS_m|}\epsilon. \tag{since \pref{event:misLoss} holds}
\end{align*}
Since $\LCB_{m,1}(a)\leq B+4\sqrt{|\calS_m|}\epsilon$ (as $a$ is not eliminated at the end of epoch $m$), we have
\begin{align*}
    \inner{a,\theta}\leq B +\rad_{m,a}^{N}+ \sum_{i=1}^{|\calS_m|}|\lambda_{m,i}^{(a)}|\cdot\left(\frac{2D\mu_{a_{m,i}}}{2^m|\calS_m|}+\frac{16\log T +2}{2^m}\right) + 8\sqrt{|\calS_m|}\epsilon.
\end{align*}
\end{proof}

\begin{lemma}\label{lem:bound_2_mis}
    If \pref{event:misLoss} holds, \pref{alg:lossLBmis} guarantees that if $a$ is not eliminated at the end of epoch $m$, then we also have
    \begin{align*}
        2^m\Delta_a\leq \frac{256n\beta^2}{\Delta_a} +\frac{4DB + 12\sum_{i=1}^{|\calS_m|}|\lambda_{m,i}^{(a)}|\cdot D\mu_{a_{m,i}}}{|\calS_m|}+(128\log T +16)\sqrt{n}+2^m\cdot 64\sqrt{n}\epsilon.
    \end{align*}
\end{lemma}
\begin{proof}
    If $\inner{a,\theta}\leq 2B$, we know that $\Delta_a = \inner{a-a^\star,\theta} \leq 2B$. Using \pref{lem:delta_1_loss_miss}, we can obtain that
    \begin{align*}
        2^m\cdot \Delta_a &\leq 2^m\cdot 24\sqrt{n}\epsilon+\frac{256n\beta^2}{\Delta_a} + \frac{2D\Delta_a}{|\calS_m|} \\
        &\leq 2^m\cdot 24\sqrt{n}\epsilon+\frac{256n\beta^2}{\Delta_a} + \frac{4DB}{|\calS_m|}
    \end{align*}
    If $\inner{a,\theta}\geq 2B$, we have $B\leq \frac{\inner{a,\theta}}{2}$. Using \pref{lem:epoch_B_with_mis}, we know that
    \begin{align*}
        \Delta_a &\leq \inner{a,\theta} \leq \underbrace{2\cdot \rad_{m,a}^{N}}_{\term{1}}+ \underbrace{2\sum_{i=1}^{|\calS_m|}|\lambda_{m,i}^{(a)}|\cdot\left(\frac{2D\mu_{a_{m,i}}}{2^m|\calS_m|}+\frac{16\log T +2}{2^m}\right) + 16\sqrt{|\calS_m|}\epsilon}_{\term{2}}.
    \end{align*}

    If $\term{1}\geq \term{2}$, we have
    \begin{align*}
        \Delta_a &\leq 4\rad_{m,a}^{N} \epsilon \leq 4\sqrt{|\calS_m|}\max_{a_m\in\calS_m}\rad_{m,a_m}^N \leq \frac{8\beta\sqrt{n}}{2^{m/2}},
    \end{align*}
    meaning that $2^m\Delta_a \leq \frac{64n\beta^2}{\Delta_a}$.
    Otherwise, we have
    \begin{align*}
        \Delta_a\leq 4\sum_{i=1}^{|\calS_m|}|\lambda_{m,i}^{(a)}|\cdot \left(\frac{2D\mu_{a_{m,i}}}{2^m|\calS_m|}+\frac{16\log T +2}{2^m}\right) + 64\sqrt{n}\epsilon,
    \end{align*}
    meaning that
    \begin{align*}
        2^m\Delta_a\leq \frac{8\sum_{i=1}^{|\calS_m|}|\lambda_{m,i}^{(a)}|\cdot D\mu_{a_{m,i}}}{|\calS_m|}+(128\log T +16)\sqrt{n}+2^m\cdot 64\sqrt{n}\epsilon.
    \end{align*}
    Combining both cases, we know that
    \begin{align*}
        2^m\Delta_a\leq \frac{256n\beta^2}{\Delta_a} +\frac{4DB + 12\sum_{i=1}^{|\calS_m|}|\lambda_{m,i}^{(a)}|\cdot D\mu_{a_{m,i}}}{|\calS_m|}+(128\log T +16)\sqrt{n}+2^m\cdot 64\sqrt{n}\epsilon.
    \end{align*}
\end{proof}

Now we are ready to prove our main result \pref{thm:lossLBmis}.
\begin{proof}[Proof of Theorem~\ref{thm:lossLBmis}]
    We analyze the regret when \pref{event:misLoss} holds, which happens with probability at least $1-\frac{2}{T^2}$. When \pref{event:misLoss} does not hold, the expected regret is bounded by $\frac{2}{T}$.
    
    We then bound the regret with a fixed choice of $B$. Combining \pref{lem:delta_1_loss_miss} and \pref{lem:epoch_B_with_mis}, if action $a$ is not eliminated at the end of epoch $m$, we have
    \begin{align*}
        2^{m}\cdot \Delta_a&\leq \frac{256n\beta^2}{\Delta_a} +\frac{4DB + 12\sum_{i=1}^{|\calS_{m}|}|\lambda_{m,i}^{(a)}|\cdot D\mu_{a_{m,i}}}{|\calS_m|}+(128\log T +16)\sqrt{n}+2^m\cdot 64\sqrt{n}\epsilon, \\
        2^m\cdot \Delta_a&\leq 2^m\cdot 24\sqrt{n}\epsilon+\frac{256n\beta^2}{\Delta_a} + \frac{2D\Delta_a}{|\calS_m|}.
    \end{align*}
    Therefore, we have
    \begin{align*}
        \Delta_a \leq \order\left(\frac{n\beta^2}{2^m\cdot \Delta_a} + \sqrt{n}\epsilon + \frac{\sqrt{n}\log T}{2^m}\right) +  \frac{1}{2^m}\min\left\{\frac{4DB+12\sum_{i=1}^{|\calS_m|}|\lambda_{m,i}^{(a)}|\cdot D\mu_{a_{m,i}}}{n}, \frac{D\Delta_a}{n}\right\}.
    \end{align*}
    Denote $\calT_B$ to be the number of rounds \pref{alg:lossLBmis} proceeds with $B$ and define $\Reg_B$ be the expected regret within $\calT_B$ rounds.
    Then, for any $\alpha_m\geq 0$,  the overall regret is then upper bounded as follows:
    \begin{align*}
        \Reg_B &\triangleq \sum_{m= 1}^{\lceil\log_2(|\calT_B|/3n\rceil}\sum_{a\in \calS_m}2^{m}\cdot\Delta_a \\
        &\leq \sum_{m=1}^{\lceil\log_2(|\calT_B|/3n\rceil}\sum_{a\in \calS_m}\mathbbm{1}\{\Delta_a> \alpha_m\}\left(\order\left(\frac{n\beta^2}{\Delta_a} + 2^m\sqrt{n}\epsilon + \sqrt{n}\log T\right) \right.\\
        &\qquad +\left.\min\left\{\frac{4DB+12\sum_{i=1}^{|\calS_{m-1}|}|\lambda_{m-1,i}^{(a)}|\cdot d(a_{m-1,i})}{n}, \frac{2D\Delta_a}{n}\right\}\right) \tag{since $a$ is not eliminated in epoch $m-1$ for all $a\in\calS_m$} \\
        &\qquad + \sum_{m\geq 1}\sum_{a\in \calS_m}\mathbbm{1}\{\Delta_a\leq \alpha_m\}2^m\Delta_a.
    \end{align*}
    Picking $\alpha_m = \beta\sqrt{\frac{n}{2^m}}$, we can obtain that
    \begin{align*}
        \Reg_B &=\sum_{m= 1}^{\lceil\log_2(|\calT_B|/3n\rceil}\sum_{a\in \calS_m}\left(\order\left(\beta\sqrt{n\cdot 2^m} +2^m\sqrt{n}\epsilon+ \sqrt{n}\log T \right) \right.\\
        &\qquad +\left.\min\left\{\frac{4DB+12\sum_{i=1}^{|\calS_{m-1}|}|\lambda_{m-1,i}^{(a)}|\cdot d(a_{m-1,i})}{n}, \frac{2D\Delta_a}{n}\right\}\right) \\
        &\leq \order\left(|\calT_B|\sqrt{n}\epsilon + \beta n\sqrt{|\calT_B|}+\sqrt{n}\log T\log(T/n)\right) \\
        &\qquad + \sum_{m= 1}^{\lceil\log_2(|\calT_B|/3n)\rceil}\sum_{a\in \calS_m}\min\left\{\frac{4DB+12\sum_{i=1}^{|\calS_{m-1}|}|\lambda_{m-1,i}^{(a)}|\cdot d(a_{m-1,i})}{n}, \frac{2D\Delta_a}{n}\right\}.
    \end{align*}
    
    On the other hand, picking $\alpha_m=0$, we have
    \begin{align*}
        \Reg_B &\leq \sum_{m=1}^{\lceil\log_2(|\calT_B|/3n\rceil}\sum_{a\in \calS_m}\left(\order\left(\frac{n\beta^2}{\Delta_{\min}} + 2^m\sqrt{n}\epsilon + \sqrt{n}\log T\right) \right.\\
        &\qquad +\left.\min\left\{\frac{4DB+12\sum_{i=1}^{|\calS_{m-1}|}|\lambda_{m-1,i}^{(a)}|\cdot d(a_{m-1,i})}{n}, \frac{2D\Delta_a}{n}\right\}\right) \\
        &\leq \order\left(\frac{n^2\beta^2\log(T/n)}{\Delta_{\min}}+\epsilon\sqrt{n}|\calT_B|+\sqrt{n}\log T\log(T/n)\right)\\
        &\qquad + \sum_{m= 1}^{\lceil\log_2(|\calT_B|/3n\rceil}\sum_{a\in \calS_m} \min\left\{\frac{4DB+12\sum_{i=1}^{|\calS_{m-1}|}|\lambda_{m-1,i}^{(a)}|\cdot d(a_{m-1,i})}{n}, \frac{2D\Delta_a}{n}\right\}.
    \end{align*}

    Using the fact that $\beta=\sqrt{2\log(KT^3)}$ and combining both bounds, we can obtain that
    \begin{align}
        \Reg_B &\leq \order\left(\min\left\{\frac{n^2\log(KT)\log(T/n)}{\Delta_{\min}}, n\sqrt{|\calT_B|\log(KT)}\right\}+\epsilon\sqrt{n}|\calT_B|\right)\nonumber\\
        &\qquad + \sum_{m= 1}^{\lceil\log_2(|\calT_B|/3n\rceil}\sum_{a\in \calS_m}\min\left\{\frac{4DB+12\sum_{i=1}^{|\calS_{m-1}|}|\lambda_{m-1,i}^{(a)}|\cdot d(a_{m-1,i})}{n}, \frac{2D\Delta_a}{n}\right\}. \label{eqn:reg_b}%
    \end{align}
    For notational convenience, let $R_B=\order\left(\min\left\{\frac{n^2\log(KT)\log(T/n)}{\Delta_{\min}},n\sqrt{|\calT_B|\log(KT)}\right\}+\epsilon\sqrt{n}|\calT_B|\right)$. To further analyze this bound, we first upper bound $\min\left\{\frac{4DB+12\sum_{i=1}^{|\calS_{m-1}|}|\lambda_{m-1,i}^{(a)}|\cdot d(a_{m-1,i})}{n}, \frac{2D\Delta_a}{n}\right\}$ by $\frac{2D\Delta_a}{n}$ and obtain that
    \begin{align}\label{eqn:reg_B_1}
        \Reg_B \leq R_B + \order\left(D\Delta_{\max}\log(T/n)\right).
    \end{align}

    On the other hand, we can also upper bound $\min\left\{\frac{4DB+12\sum_{i=1}^{|\calS_{m-1}|}|\lambda_{m-1,i}^{(a)}|\cdot d(a_{m-1,i})}{n}, \frac{2D\Delta_a}{n}\right\}$ by $\frac{12\sum_{i=1}^{|\calS_{m-1}|}|\lambda_{m-1,i}^{(a)}|\cdot d(a_{m-1,i})}{n}$ and obtain that
    \begin{align*}
        \Reg_B \leq R_B + \left(\sum_{m=1}^{\lceil\log_2(|\calT_B|/3n\rceil}\sum_{a\in\calS_m}\frac{4DB+12\sum_{i=1}^{|\calS_{m-1}|}|\lambda_{m-1,i}^{(a)}|\cdot d(a_{m-1,i})}{n}\right).
    \end{align*} 

    Let $L_{\Alg}^m=\sum_{a\in \calS_m}2^m\mu_a$ be the total expected loss within epoch $m$ and $L_{\star}^m=|\calS_m|\cdot 2^m\cdot\mu^\star$ be the total expected loss for the optimal action. Define $\Reg_m=L_{\Alg}^m-L_{\star}^m$. Direct calculation shows that
    \begin{align*}
        &\sum_{a\in \calS_m}\frac{\sum_{i=1}^{|\calS_{m-1}|}|\lambda_{m-1,i}^{(a)}|\cdot d(a_{m-1,i})}{n} \\
        &\leq \frac{3D}{2^{m-1}}\cdot 2^{m-1}\sum_{i=1}^{|\calS_{m-1}|}\mu_{a_{m-1,i}}\tag{since $|\lambda_{m-1,i}^{(a)}|\leq 1$ and $|\calS_m|=3n$}\\
        &= \frac{3D}{2^{m-1}}L_{\Alg}^{m-1}.
    \end{align*}
    Using the fact that $\Reg_B = \sum_{m=1}^{\lceil\log(|\calT_B|/3n)\rceil}\Reg_m$, we know that
    \begin{align*}
        &\sum_{m=1}^{\lceil\log(|\calT_B|/3n)\rceil} (L_{\Alg}^m-L^m_{\star})\\
        &\leq \sum_{m=1}^{\lceil\log(|\calT_B|/3n)\rceil}\Reg_m + 2\epsilon\cdot |\calT_B|\\
        &\leq R_B + \sum_{m=\lceil\log_2(72D)\rceil+1}^{\lceil\log(|\calT_B|/3n)\rceil}\frac{36D}{2^{m-1}}\cdot L_{\Alg}^{m-1} + \sum_{m=1}^{\lceil\log_2(72D)\rceil}2^m\Delta_{\max} + 12DB\log(T/n) \tag{$2\epsilon\cdot |\calT_B|$ is subsumed in $R_B$}\\
        &\leq R_B + \sum_{m=\lceil\log_2(72D)\rceil+1}^{\lceil\log(|\calT_B|/3n)\rceil}\frac{36D}{2^{m-1}}\cdot \left(L_{\Alg}^{m-1}-L_{\star}^{m-1}\right) + \sum_{m=\lceil\log_2(72D)\rceil+1}^{\lceil\log(|\calT_B|/3n)\rceil}\frac{36D}{2^{m-1}}\cdot L_{\star}^{m-1} + \sum_{m=1}^{\lceil\log_2(72D)\rceil}2^m\Delta_{\max} +12DB\log(T/n)\\
        &\leq R_B + \frac{1}{2}\sum_{m=\lceil\log_2(72D)\rceil+1}^{\lceil\log(|\calT_B|/3n)\rceil} \left(L_{\Alg}^{m-1}-L_{\star}^{m-1}\right) + 36nD\mu^\star\log(T/(216nD))+144D\Delta_{\max} +12DB\log(T/n)\\
        &= R_B + \frac{1}{2}\sum_{m=\lceil\log_2(72D)\rceil+1}^{\lceil\log(|\calT_B|/3n)\rceil} \left(L_{\Alg}^{m-1}-L_{\star}^{m-1}\right) + 36nd^\star\log(T/(216nD))+144D\Delta_{\max}+12DB\log(T/n).
    \end{align*}
    Rearranging the terms, we can obtain that
    \begin{align}\label{eqn:reg_B_2}
        \Reg_B \leq R_B + 72nd^\star\log(T/(216nD))+288D\Delta_{\max}+12DB\log(T/n).
    \end{align}
    Combining \pref{eqn:reg_B_1} and \pref{eqn:reg_B_2}, we know that
    \begin{align}
        \Reg_B &\leq \order\left(\min\left\{\frac{n^2\log(KT)\log(T/n)}{\Delta_{\min}},n\sqrt{|\calT_B|\log(KT)}\right\}+\epsilon\sqrt{n}|\calT_B|\right) \nonumber \\
        &\qquad +\order\left( \min\left\{nd^\star\log (T/nD)+D\Delta_{\max}+DB\log(T/n),D\Delta_{\max}\log (T/n)\right\}\right).\label{eqn:reg_B_final}
    \end{align}
    Finally, according to \pref{lem:end-of-B}, \pref{alg:lossLBmis} fails at most $\lceil\log_2(D\mu^\star))\rceil = \lceil\log_2(d^\star))\rceil$ times. Summing up the regret over all rounds, we know that the overall regret is bounded as follows
    \begin{align*}
        \Reg \leq \sum_{r=0}^{\lceil\log_2(d^\star))\rceil}\Reg_{2^r/D} &\leq \order\left(\min\left\{\frac{n^2\log(KT)\log(T/n)\log(d^\star)}{\Delta_{\min}},n\sqrt{T\log(d^\star)\log(KT)}\right\}+\epsilon\sqrt{n}T\right) \\
        &\qquad + \log(d^\star)\cdot \order\left( \min\left\{nd^\star\log (T/n)+D\Delta_{\max},D\Delta_{\max}\log (T/n)\right\}\right),
    \end{align*}
	which finishes the proof.
\end{proof}
\section{Omitted Details for Delay-as-Reward}\label{app: reward}
In this section, we show our results for the delay-as-reward setting. The difference compared with the delay-as-loss setting is that now, $\mu_a=\inner{a,\theta}+\epsilon_a\in[0,1]$ represents the expected reward of picking action $a$, where $|\epsilon_a|\leq \epsilon$ for all $a\in\calA$. The learner's goal is to minimize the pseudo regret defined as follows:
\begin{align}\label{eqn:reward-regret}
    \Reg\triangleq T\max_{a\in\calA}\inner{a,\theta} - \E\left[\sum_{t=1}^T\inner{a_t,\theta}\right].
\end{align}
 Define $\Delta_a = \inner{a^\star-a,\theta}$ as the suboptimality gap of action $a$, where $a^\star \in \argmax_{a\in\calA}\inner{a, \theta}$, and $\mu^\star \triangleq \max_{a\in\calA}\mu_a$ as the reward of the optimal action. Again, note that due to the misspecification, $\mu^\star$ may not necessarily be $\mu_{a^\star}$. Define $\Delta_{\min} = \min_{a\in \calA, \Delta_a>0}\Delta_a$ to be the minimum non-zero sub-optimality gap. The delay at round $t$ is still defined as $d_t=D\cdot u_t$, and $d^\star = D\cdot \mu^\star$ is  the expected delay of the optimal action. We also define $d(a)=D\mu_a$ to be the expected delay for action $a$.

\newpage
\subsection{Algorithm for Linear Bandits with Delay-as-Reward}
We list our algorithm for the reward case in \pref{alg:rewardLBmis} for completeness. The algorithm shares the same idea as \pref{alg:lossLBmis}.

\setcounter{AlgoLine}{0}
\begin{algorithm}[H]
\caption{Phased Elimination for Linear Bandits with Delay-as-Reward}\label{alg:rewardLBmis}

\nl Input: maximum possible delay $D$, action set $\calA$, $\beta>0$, a misspecification level $\epsilon$. 

\nl Initialize optimal reward guess $B=1$.

\nl Initialize active action set $\calA_1=\calA$.  \label{line:reward-restart} 

\nl \For{$m=1,2,\dots,$}{
    \nl Find $\calS_m=\{a_{m,1},\dots,a_{m,|\calS_m|}\}$ to be the volumetric spanner of $\calA_m$, where $|\calS_m|= 3n$. \label{line:volume-reward}
    
    \nl Pick each $a\in \calS_m$ $2^m$ times in a round-robin way. \label{line:round-robin-reward}

    \nl Let $\calI_m$ contain all the rounds in this epoch.
    
    \nl For all $a\in \calS_m$, calculate the following quantities
    \begin{align}
        &\hat{\mu}_{m}^+(a)=\frac{1}{2^m}\Big(\sum_{\tau\in \obs_m(a)}u_{\tau} + \sum_{\tau\in \unobs_m(a)}1\Big), \\
        &\hat{\mu}_{m}^-(a)=\frac{1}{2^m}\sum_{\tau\in \obs_m(a)}u_{\tau}, \\
        &\hat{\mu}_{m,1}^{+}(a)=\hat{\mu}^{+}_{m}(a)+\frac{\beta}{2^{m/2}}\|a\|_2, \label{eqn:reward-ucb-linear-1-mis}\\
        &\hat{\mu}_{m,1}^{-}(a)=\hat{\mu}^{-}_{m}(a)-\frac{\beta}{2^{m/2}}\|a\|_2,\label{eqn:reward-lcb-linear-1-mis}\\
        &\hat{\mu}_m^{F}(a)=\frac{1}{\unbiasSize_m(a)}\sum_{\tau\in \unbias_m(a)}u_{\tau},\\
        &\hat{\mu}_{m,2}^{+}(a)=\hat{\mu}_m^F(a)+\frac{\beta}{\sqrt{\unbiasSize_m(a)}}\|a\|_2, \label{eqn:reward-ucb-linear-2-mis}\\
        &\hat{\mu}_{m,2}^{-}(a)=\hat{\mu}_m^F(a)-\frac{\beta}{\sqrt{\unbiasSize_m(a)}}\|a\|_2, \label{eqn:reward-lcb-linear-2-mis}
    \end{align}
    where $\unbiasSize_m(a) = |\unbias_m(a)|$, $\unbias_m(a) = \{\tau\in \calI_m: \tau+D\in\calI_m, a_{\tau}=a\}$, $\obs_m(a) = \{\tau\in \calI_m: \tau+d_{\tau}\in\calI_m, a_{\tau}=a\}$, and
    $\unobs_m(a)= \{\tau\in \calI_m: a_{\tau}=a\}\setminus\obs_m(a)$.

    \nl \For{each $a\in \calA_m$}{
        \nl \label{line: decompose-rewward}
        Decompose $a$ as $a=\sum_{i=1}^{|S_m|}\lambda_{m,i}^{(a)}a_{m,i}$ with $\|\lambda_{m}^{(a)}\|_2\leq 1$ and calculate 
        {\small
        \begin{align}
            &\LCB_{m}(a)=\sum_{i=1}^{|\calS_m|}\lambda_{m,i}^{(a)}\cdot\hat{\mu}_{m,2}^{-\sgn(\lambda_{m,i}^{(a)})}(a_{m,i}), \label{eqn:reward-ucb-f-all-action-mis} \\
            &\UCB_m(a) = \max_{j\in \{1,2\}}\{\UCB_{m,j}(a)\} \;\;\text{where} \nonumber  \\
            & \UCB_{m,j}(a)=\sum_{i=1}^{|\calS_m|}\lambda_{m,i}^{(a)}\cdot\hat{\mu}_{m,j}^{\sgn(\lambda_{m,i}^{(a)})}(a_{m,i}),\label{eqn:reward-lcb-f-all-action-mis}
        \end{align}
        }
    }
    
    \nl Set $\calA_{m+1} = \calA_m$.
    
    \nl \For{$a_1\in \calA_m$}{
        \nl \label{line:reward_eliminate_miss}\If{$\exists a_2\in \calA_m$, such that $\max\{\LCB_m(a_2),B\} \geq \UCB_m(a_1)+4\sqrt{3n}\epsilon $}
        {
         \nl Eliminate $a_1$ from $\calA_{m+1}$.
        }
    }
    \nl \If{$\calA_{m+1}$ is empty}{
        Set $B\leftarrow B/2$ and go to \pref{line:reward-restart}.
    }
}
\end{algorithm}

\subsection{Regret Guarantees}\label{app: reward-regret}
In this section, we show the theoretical guarantees for our algorithm in the delay-as-reward setting.

\begin{theorem}\label{thm:reward-main}
    \pref{alg:rewardLBmis} with $\beta = \sqrt{2\log(KT^3)}$ guarantees that
    \begin{align*}
        \Reg &\leq \order\left(\min\left\{\frac{n^2\log(KT)\log(T/n)\log(1/\mu^\star)}{\Delta_{\min}}, n\sqrt{T\log(KT)\log(1/\mu^\star)}\right\}+\epsilon\sqrt{n}T\right)\\
        &\qquad+\order\left(\min\left\{\sum_{j=0}^{\lceil\log_2(1/\mu^\star)\rceil}\sum_{m=1}^{\lceil\log_2(|\calT_{2^{-j}}|/3n\rceil}\sum_{i=1}^{3n} d(a_{m-1,i}^{(2^{-j})}), D\Delta_{\max}\log(1/\mu^\star)\log(T/n)\right\}\right),
    \end{align*}
    where $\{a_{m,i}^{(B)}\}_{i=1}^{3n}$ represents the set of volumetric spanner at epoch $m$ with the optimal reward guess $B$. 
\end{theorem}

Similar to the analysis in \pref{app:loss}, our analysis is based on the condition that \pref{event:misLoss} holds, which happens with probability $1-\frac{2}{T^2}$ according to \pref{lem:high-prob-event}. The following lemma is a counterpart of \pref{lem:end-of-B}, providing an upper bound of the number of guesses on the optimal reward $B$.

\begin{lemma}\label{lem:end-of-B-reward}
    Suppose that \pref{event:misLoss} holds. If $B\leq \mu^\star$, then $a^\star\in \calA_m$ for all $m$.
\end{lemma}
\begin{proof}
    Since \pref{event:misLoss} holds, we have, we know that for all $a\in\calA_m$, $\UCB_m(a)+\sqrt{|\calS_m|}\epsilon\geq \inner{a,\theta}$, $\LCB_m(a)+\sqrt{|\calS_m|}\epsilon\leq \inner{a,\theta}$
    If $B\leq \mu^\star$, then we have $a^\star$ never eliminated since for any $a\in\calA_m$,
    \begin{align*}
         \UCB_{m}(a^\star) +2\epsilon\sqrt{|\calS_m|} &\geq \max_{a\in\calA}\{\inner{a,\theta}+\epsilon_a\} \geq \mu^\star \geq B,\\
         \UCB_{m}(a^\star) +4\epsilon\sqrt{|\calS_m|} &\geq \mu^\star + 2\epsilon\sqrt{|\calS_m|} \geq \inner{a,\theta}  + \epsilon\sqrt{|\calS_m|}\geq \LCB_m(a).
    \end{align*}
    Therefore, $a^\star$ never satisfy the elimination condition.
\end{proof}

The following lemma is a counterpart of \pref{lem:delta_1_loss_miss}.

\begin{lemma}\label{lem:delta_1_reward_miss}
    Suppose that \pref{event:misLoss} holds. \pref{alg:rewardLBmis} guarantees that if $a\in\calA$ is not eliminated at the end of epoch $m$ (meaning that $a\in \calA_{m+1}$), then 
    \begin{align*}
        2^m\cdot \Delta_a\leq 2^m\cdot 24\sqrt{n}\epsilon+\frac{256n\beta^2}{\Delta_a} + \frac{2D\Delta_a}{|\calS_m|}.
    \end{align*}
\end{lemma}
\begin{proof}
    Since \pref{event:misLoss} holds, we know that for all $a\in\calA_m$, $\LCB_m(a)\leq \mu_a + \sqrt{|\calS_m|}\epsilon$, $\UCB_m(a)\geq \mu_a - \sqrt{|\calS_m|}\epsilon$. Moreover, as $\UCB_m(a)=\min\{\UCB_{m,1}(a),\UCB_{m,2}(a)\}$, we know that for all $a\in \calA_m$
    \begin{align*}
        \UCB_{m,1}(a) - 2\rad_{m,a}^{N} - 2\epsilon\sqrt{|\calS_m|}=  \hat{\mu}_{m,1}(a) - \rad_{m,a}^{N} - 2\epsilon\sqrt{|\calS_m|}\leq \inner{a,\theta},\\
        \UCB_{m,2}(a) - 2\rad_{m,a}^{F} - 2\epsilon\sqrt{|\calS_m|}=  \hat{\mu}_{m,2}(a) - \rad_{m,a}^{F} - 2\epsilon\sqrt{|\calS_m|}\leq \inner{a,\theta},\\
        \LCB_{m}(a) + 2\rad_{m,a}^{F} + 2\epsilon\sqrt{|\calS_m|} = \hat{\mu}_{m,2}(a) + \rad_{m,a}^{F}+2\epsilon\sqrt{|\calS_m|}\geq \inner{a,\theta}.        
    \end{align*}
    If $B\leq \mu^\star$, then $a^\star\in \calA_m$ according to \pref{lem:end-of-B-reward}.
    Moreover, if $a$ is not eliminated in epoch $m$, we have $\UCB_m(a)+4\sqrt{|S_m|}\epsilon\geq \max\{\LCB_m(a^\star),B\}$, meaning that
    \begin{align*}
        &\inner{a,\theta} + 2\rad_{m,a}^{F} + 2\epsilon\sqrt{|\calS_m|} \\
        &\geq \wh{\mu}_{m,2}(a) + \rad_{m,a}^{F} \\
        &\geq \UCB_m(a) \\
        &\geq \max\{\LCB_m(a^\star),B\}-4\sqrt{|S_m|}\epsilon \\
        &\geq \LCB_m(a^\star) - 4\sqrt{|S_m|}\epsilon \\
        &= \wh{\mu}_{m,2}(a^\star) - \rad_{m,a^\star}^{F} - 4\sqrt{|S_m|}\epsilon \\
        &\geq \inner{a^\star,\theta}  - 2\rad_{m,a^\star}^{F} - 6\sqrt{|S_m|}\epsilon.
    \end{align*}
    Since $\rad_{m,a}^F = \sum_{i=1}^{|\calS_m|}|\lambda_{m,i}^{(a)}|\cdot \rad_{m,a_{m,i}}^{F}$ with $\|\lambda_{m}^{(a)}\|_2\leq 1$, we have that $\|\lambda_{m}^{(a)}\|_1\leq \sqrt{|\calS_m|}$ and
    \begin{align*}
        &\Delta_a\leq 4\sqrt{|\calS_m|}\left(\max_{a\in S_m}\rad_{m,a}^{F}+2\epsilon\right)= 4\sqrt{3n}\max_{a\in S_m}\rad_{m,a}^{F}+8\sqrt{3n}\epsilon \leq \frac{8\sqrt{n}\beta}{\min_{a'\in \calS_m}\sqrt{\unbiasSize_m(a')}}+16\sqrt{n}\epsilon.
    \end{align*}
    
    If $B\geq \mu^\star$, then we have
    \begin{align*}
        \mu^\star\leq B \leq \UCB_{m}(a) + 4\sqrt{|\calS_m|}\epsilon \leq \mu_a + 2\rad_{m,a}^{F} + 6\sqrt{|\calS_m|}\epsilon,
    \end{align*}
    where the second inequality is because $a$ is not eliminated in epoch $m$. Therefore, we always have
    \begin{align*}
        \Delta_a &\leq 2\rad_{m,a}^{F} + 6\sqrt{|\calS_m|}\epsilon \leq \frac{8\sqrt{n}\beta}{\min_{a'\in \calS_m}\sqrt{\unbiasSize_m(a')}} + 12\sqrt{n}\epsilon.
    \end{align*} 
    In addition, we know that for all $a\in \calS_m$,
    \begin{align*}
        2^m &= |\calS_m| \leq \unbiasSize_m(a) + \frac{D}{|S_m|} + 1 \leq \unbiasSize_m(a) + \frac{2D}{|S_m|}.
    \end{align*}
    Therefore, if $12\sqrt{n}\epsilon\geq \frac{\Delta_a}{2}$, then we have
    \begin{align*}
        2^m\Delta_a\leq 2^m\cdot 24\sqrt{n}\epsilon;
    \end{align*}
    otherwise, we have $\Delta_a \leq \frac{8\sqrt{n}\beta}{\min_{a\in \calS_m}\sqrt{\unbiasSize_m(a)}} + 12\sqrt{n}\epsilon \leq \frac{8\sqrt{n}\beta}{\min_{a\in S_m}\sqrt{\unbiasSize_m(a)}}  + \frac{\Delta_a}{2}$ and
    \begin{align*}
        \Delta_a \leq \frac{16\sqrt{n}\beta}{\min_{a'\in \calS_m}\sqrt{\unbiasSize_m(a')}},
    \end{align*}
    and we can obtain that
    \begin{align*}
        \min_{a'\in S_m}{\unbiasSize_m(a')}\cdot \Delta_a\leq \frac{256d\beta^2}{\Delta_a}.
    \end{align*}
    Combining the above two cases, we know that for all $a\in\calA_m$, $$2^m\cdot \Delta_a\leq 2^m\cdot 24\sqrt{n}\epsilon+ \min_{a'\in \calS_m}\unbiasSize_m(a')\cdot \Delta_a + \frac{2D\Delta_a}{|\calS_m|} \leq  2^m\cdot 24\sqrt{n}\epsilon+\frac{256n\beta^2}{\Delta_a} + \frac{2D\Delta_a}{|\calS_m|}.$$
\end{proof}

The following lemma is a counterpart of \pref{lem:epoch_B_with_mis}. 

\begin{lemma}\label{lem:epoch_B_with_mis-reward}
    \pref{alg:rewardLBmis} guarantees that under \pref{event:misLoss}, if an action $a$ is eliminated at the end of epoch $m$ (meaning that $a\in \calA_m$), then
    \begin{align*}
    B\leq \inner{a,\theta} +\rad_{m,a}^{N}+ \sum_{i=1}^{|\calS_m|}|\lambda_{m,i}^{(a)}|\cdot\left(\frac{2d(a_{m,i})}{2^m|\calS_m|}+\frac{16\log T +2}{2^m}\right) + 8\sqrt{|\calS_m|}\epsilon,
\end{align*}
where $d(a)=D\mu_a$.
\end{lemma}
\begin{proof}
Under \pref{event:misLoss}, we know that for all $a\in \calA_m$,
\begin{align*}
    \inner{a,\theta} &= \sum_{i=1}^{|\calS_m|}\lambda_{m,i}^{(a)}\inner{a_{m,i},\theta^\star} \\
    &= \sum_{i=1}^{|\calS_m|}\lambda_{m,i}^{(a)}(\mu_{a_{m,i}}-\epsilon_{a_{m,i}}) \tag{since $\mu_a = \inner{a,\theta^\star}+\epsilon_a$}\\
    &\geq \sum_{i=1}^{|\calS_m|}\lambda_{m,i}^{(a)}\cdot \mu_{a_{m,i}} - \sqrt{|\calS_m|}\epsilon \tag{since $\|\lambda_{m}^{(a)}\|_1\leq \sqrt{|\calS_m|}$} \\
    &\geq \sum_{i=1}^{|\calS_m|}\lambda_{m,i}^{(a)}\cdot \hat{\mu}_{m}(a_{m,i}) - \rad_{m,a}^{N} - 3\sqrt{|\calS_m|}\epsilon \tag{since \pref{event:misLoss} holds}\\
    &\geq \sum_{i=1}^{|\calS_m|}\lambda_{m,i}^{(a)}\cdot\hat{\mu}_{m}^{sgn(\lambda_{m,i}^{(a)})}(a_{m,i}) - \rad_{m,a}^{N} -\sum_{i=1}^{|\calS_m|}|\lambda_{m,i}^{(a)}|\cdot \frac{|\unobs_m(a_{m,i})|}{2^m} - 3\sqrt{|\calS_m|}\epsilon \tag{using \pref{eqn:pos-bias} and \pref{eqn:neg-bias}}\\
    &= \UCB_{m,1}(a) - \rad_{m,a}^{N} -\sum_{i=1}^{|\calS_m|}|\lambda_{m,i}^{(a)}|\cdot \frac{|\unobs_m(a_{m,i})|}{2^m} - 3\sqrt{|\calS_m|}\epsilon\\
    &\geq \UCB_{m,1}(a) - \rad_{m,a}^{N} - \sum_{i=1}^{|\calS_m|}|\lambda_{m,i}^{(a)}|\cdot\left(\frac{2d(a_{m,i})}{2^m|\calS_m|}+\frac{16\log KT +2}{2^m}\right) - 4\sqrt{|\calS_m|}\epsilon. \tag{since \pref{event:misLoss} holds}
\end{align*}
Since $\UCB_{m,1}(a)\geq B - 4\sqrt{|\calS_m|}\epsilon$ (as $a$ is not eliminated at the end of epoch $m$), we have
\begin{align*}
    B\leq \inner{a,\theta} +\rad_{m,a}^{N}+ \sum_{i=1}^{|\calS_m|}|\lambda_{m,i}^{(a)}|\cdot\left(\frac{2d(a_{m,i})}{2^m|\calS_m|}+\frac{16\log T +2}{2^m}\right) + 8\sqrt{|\calS_m|}\epsilon.
\end{align*}
\end{proof}

The following lemma is a counterpart of \pref{lem:bound_2_mis}.

\begin{lemma}\label{lem:bound_2_mis_reward}
    If $B\geq \frac{\mu^\star}{2}$ and \pref{event:misLoss} holds, \pref{alg:rewardLBmis} guarantees that if $a$ is not eliminated at the end of epoch $m$, then we also have
    \begin{align*}
        2^m\Delta_a\leq \frac{256n\beta^2}{\Delta_a} +\frac{12\sum_{i=1}^{|\calS_m|}|\lambda_{m,i}^{(a)}|\cdot d(a_{m,i})}{|\calS_m|}+(128\log T +16)\sqrt{n}+2^m\cdot 64\sqrt{n}\epsilon,
    \end{align*}
    where $d(a)=D\mu_a$.
\end{lemma}
\begin{proof}
    If $\inner{a,\theta}\geq \frac{B}{2}$, we know that $\Delta_a = \inner{a^\star-a,\theta} \leq 3\inner{a,\theta}$. Using \pref{lem:delta_1_reward_miss}, we can obtain that
    \begin{align*}
        2^m\cdot \Delta_a &\leq 2^m\cdot 24\sqrt{n}\epsilon+\frac{256n\beta^2}{\Delta_a} + \frac{2D\inner{a,\theta}}{|\calS_m|} \\
        &\leq 2^m\cdot 24\sqrt{n}\epsilon+\frac{256n\beta^2}{\Delta_a} + \frac{2\sum_{i=1}^{|\calS_m|}|\lambda_{m,i}^{(a)}|\cdot d(a_{m,i})}{|\calS_m|}.
    \end{align*}
    If $\inner{a,\theta} \leq \frac{B}{2}$, we have $3(B-\inner{a,\theta} ) \geq \frac{3B}{2} \geq \inner{a^\star-a,\theta}$. Using \pref{lem:epoch_B_with_mis}, we know that
    \begin{align*}
        \Delta_a &\leq \mu_a \leq \underbrace{3\cdot \rad_{m,a}^{N}}_{\term{1}}+ \underbrace{3\sum_{i=1}^{|\calS_m|}|\lambda_{m,i}^{(a)}|\cdot\left(\frac{2d(a_{m,i})}{2^m|\calS_m|}+\frac{16\log T +2}{2^m}\right) + 24\sqrt{|\calS_m|}\epsilon}_{\term{2}}.
    \end{align*}

    If $\term{1}\geq \term{2}$, we have
    \begin{align*}
        \Delta_a &\leq \mu_a \leq 6\rad_{m,a}^{N} \epsilon \leq 6\sqrt{|\calS_m|}\max_{a_m\in\calS_m}\rad_{m,a_m}^N \leq \frac{12\beta\sqrt{n}}{2^{m/2}},
    \end{align*}
    meaning that $2^m\Delta_a \leq \frac{144n\beta^2}{\Delta_a}$.
    Otherwise, we have
    \begin{align*}
        \Delta_a\leq 6\sum_{i=1}^{|\calS_m|}|\lambda_{m,i}^{(a)}|\cdot \left(\frac{2d(a_{m,i})}{2^m|\calS_m|}+\frac{16\log T +2}{2^m}\right) + 96\sqrt{n}\epsilon,
    \end{align*}
    meaning that
    \begin{align*}
        2^m\Delta_a\leq \frac{12\sum_{i=1}^{|\calS_m|}|\lambda_{m,i}^{(a)}|\cdot d(a_{m,i})}{|\calS_m|}+(96\log T +12)\sqrt{n}+2^m\cdot 96\sqrt{n}\epsilon.
    \end{align*}
    Combining both cases, we know that
    \begin{align*}
        2^m\Delta_a\leq \frac{256n\beta^2}{\Delta_a} +\frac{12\sum_{i=1}^{|\calS_m|}|\lambda_{m,i}^{(a)}|\cdot d(a_{m,i})}{|\calS_m|}+(96\log T +12)\sqrt{n}+2^m\cdot 96\sqrt{n}\epsilon.
    \end{align*}
\end{proof}

Now we are ready to prove our main result \pref{thm:reward-main}.
\begin{proof}[Proof of Theorem~\ref{thm:reward-main}]
Combining \pref{lem:delta_1_reward_miss} and \pref{lem:bound_2_mis_reward} and following the exact same process of obtaining \pref{eqn:reg_b} in \pref{thm:lossLBmis}, we can obtain that for a fixed value of $B$, \pref{alg:rewardLBmis} guarantees that
    \begin{align*}
        \Reg_B &\leq \order\left(\min\left\{\frac{n^2\log(KT)\log(T/n)}{\Delta_{\min}}, n\sqrt{|\calT_B|\log(KT)}\right\}+\epsilon\sqrt{n}|\calT_B|\right)\\
        &\qquad + \sum_{m= 1}^{\lceil\log_2(|\calT_B|/3n\rceil}\sum_{a\in \calS_m}\order\left(\min\left\{\frac{\sum_{i=1}^{|\calS_{m-1}|}|\lambda_{m-1,i}^{(a)}|\cdot d(a_{m-1,i})}{n}, \frac{D\Delta_a}{n}\right\}\right) \\
        &\leq \order\Bigg(\min\left\{\frac{n^2\log(KT)\log(T/n)}{\Delta_{\min}}, n\sqrt{|\calT_B|\log(KT)}\right\}+\epsilon\sqrt{n}|\calT_B|\\
        &\qquad \left.+\min\left\{\sum_{m= 1}^{\lceil\log_2(|\calT_B|/3n\rceil}\sum_{i=1}^{|\calS_{m-1}|} d(a_{m-1,i}), D\Delta_{\max}\log(T/n)\right\}\right).
    \end{align*}
    According to \pref{lem:end-of-B-reward}, there are at most $\lceil\log_2(1/\mu^\star)\rceil$ different values of $B$. With an abuse of notation, we define $\calS_{m}^{(B)}=\{a_{m,i}^{(B)}\}_{i\in [3n]}$ to be the volumetric spanner at epoch $m$ with the reward guess $B$.
    Taking summation over all these values, we can obtain that
    \begin{align*}
        \Reg &\leq \order\left(\min\left\{\frac{n^2\log(KT)\log(T/n)\log(1/\mu^\star)}{\Delta_{\min}}, n\sqrt{T\log(KT)\log(1/\mu^\star)}\right\}+\epsilon\sqrt{n}T\right)\\
        &\qquad+\order\left(\min\left\{\sum_{j=0}^{\lceil\log_2(1/\mu^\star)\rceil}\sum_{m=1}^{\lceil\log_2(|\calT_{2^{-j}}|/3n\rceil}\sum_{i=1}^{3n} d(a_{m-1,i}^{(2^{-j})}), D\Delta_{\max}\log(1/\mu^\star)\log(T/n)\right\}\right),
    \end{align*}
    completing the proof.
\end{proof}    

    While we can further apply a similar analysis to the one in \pref{thm:lossLBmis} to bound the term $\sum_{j=0}^{\lceil\log_2(1/\mu^\star)\rceil}\sum_{m=1}^{\lceil\log_2(|\calT_{2^{-j}}|/3n\rceil}\sum_{i=1}^{3n} d(a_{m-1,i}^{(2^{-j})})$ and obtain a bound with respect to $d^\star$, since $d^\star\geq D\Delta_{\max}+\epsilon$, this $d^\star$ dependent bound does not provide a significantly better regret guarantee in the worst case. This  difference in loss versus reward is also pointed out in \citep{schlisselberg2024delay} in the MAB setting. We keep this term in the upper bound since this quantity can still be potentially smaller than $D\Delta_{\max}\log(1/\mu^\star)\log(T/n)$.

\section{Omitted Details in \pref{sec: contextual}}\label{app: contextual}
In this section, we provide the omitted details in \pref{sec: contextual}.
We start with the following lemma that is a standard application of the Azuma-Hoeffding's inequality.
\begin{lemma}[Proposition 2 in \citep{hanna2023contexts}]\label{lem:prop_two}
    For each epoch $m$, \pref{alg:reduction} guarantees that with probability at least $1-\frac{\delta}{T}$, the following holds:
    \begin{align*}
        \left|\inner{g(\theta),\theta'} - \inner{g^{(m)}(\theta),\theta'}\right| \leq 2\sqrt{\frac{\log(2T|\Theta'|/\delta)}{2^{m-1}}},~~\forall \theta,\theta'\in\Theta'.
    \end{align*}
\end{lemma}

Next, we provide the proof for \pref{thm:reduction}.
\begin{proof}[Proof of Theorem~\ref{thm:reduction}]
    Define $\theta_0 = \argmin_{\theta'\in \Theta'}\|\theta'-\theta\|_2$. 
    Following the analysis of \citet{hanna2023contexts}, we decompose the regret $\Reg_m$ within epoch $m$ as follows:
    \begin{align*}
        \Reg_m &=\E\left[\sum_{\tau=2^{m-1}+1}^{2^m}\left(\inner{\argmin_{a\in\calA_t}\inner{a,\theta_t},\theta} - \min_{a_\tau^\star\in\calA_\tau}\inner{a_{\tau}^\star,\theta}\right)\right] \\
        &\leq \E\left[\sum_{\tau=2^{m-1}+1}^{2^m}\left(\inner{\argmin_{a\in\calA_t}\inner{a,\theta_t},\theta_0} - \min_{a_\tau^\star\in\calA_\tau}\inner{a_{\tau}^\star,\theta_0}\right)\right] + \order\left(\frac{2^{m-1}}{T}\right)\\
        &= \E\left[\sum_{\tau=2^{m-1}+1}^{2^m}\inner{g(\theta_t)-g(\theta_0),\theta_0}\right] + \order\left(\frac{2^{m-1}}{T}\right)\\
        &= \E\underbrace{\left[\sum_{\tau=2^{m-1}+1}^{2^{m}}\inner{g(\theta_t)-g^{(m)}(\theta_t),\theta_0}\right]}_{\Err{1}}  + \E\underbrace{\left[\sum_{\tau=2^{m-1}+1}^{2^m}\inner{g^{(m)}(\theta_t)-g^{(m)}(\theta_0),\theta_0}\right]}_{\regnctx} \\
        &\qquad + \underbrace{\E\left[\sum_{\tau=2^{m-1}+1}^{2^m}\inner{g^{(m)}(\theta_0)-g(\theta_0),\theta_0}\right]}_{\Err{2}} +\order\left(\frac{2^{m-1}}{T}\right),
    \end{align*}
    where the second equality is because
    $\E\left[\min_{a\in\calA_t}\inner{a,\theta_0}\right] = \E\left[\inner{\argmin_{a\in\calA_t}\inner{a,\theta_0},\theta_0}\right] = \inner{g(\theta_0), \theta
    _0}$.
    
    For $\Err{1}$ and $\Err{2}$, we apply \pref{lem:prop_two} to bound both terms  by $\order\left(\sqrt{2^m\log(T|\Theta'|)}\right)$.
    As for $\regnctx$, this is in fact the regret of misspecified non-contextual linear bandits with action set $\calX_m$ and misspecification level $\max_{\theta'\in \Theta'}\left|\inner{g^{(m)}(\theta')-g(\theta'),\theta}\right|$, since $\E[u_t]=\inner{g(\theta_t),\theta}$ for all $t$. Applying \pref{lem:prop_two} again, we know that the misspecification is of order $\epsilon_m=\order(\sqrt{\log(T|\Theta'|)/2^m})$ with probability at least $1-\frac{1}{T^2}$. 
    Then, applying the regret guarantee of \pref{alg:lossLBmis} proven in \pref{thm:lossLBmis}, we know that
    \begin{align*}
        &\regnctx 
        \leq \order\left(\sqrt{n2^m\log(T|\Theta'|)}\right)
        +\order\left(\min\{V_{m,1},V_{m,2}\},\log(\overline{d}^\star)\min\{W_{m,1},W_{m,2}\}\right),
    \end{align*}
    where $V_{m,1}=\frac{n^2\log(T|\Theta'|)\log(T/n)\log(\overline{d}^\star)}{\Delta_{\min}^{\nctx}}$, $V_{m,2}=n\sqrt{2^m\log(\overline{d}^\star)\log(T|\Theta'|)}$, 
    $W_{m,1}=n\overline{d}^\star\log(T/n)+D\Delta_{\max}^{\nctx}$, and $W_{m,2}=D\Delta_{\max}^{\nctx}\log(T/n)$.
    Taking a summation over all $m\in[\lceil\log_2(T)\rceil]$ epochs and using the fact that $|\Theta'|\leq \order(T^n)$ finishes the proof.
\end{proof}
\newpage
\section{Omitted Details in \pref{sec: experiment}}\label{app: experiment}
For completeness, we include the pseudo code for the benchmark used in our experiment, that is, \texttt{LinUCB} using only the observed feedback;
see \pref{alg:linUCB}.

\begin{algorithm}[h]
\caption{LinUCB with Delayed Feedback}\label{alg:linUCB}
Input: action set $\calA$, a parameter $\lambda>0$.

Initialize: $\wh{\theta}_1$ arbitrarily, $\beta_t = \sqrt{\lambda} + \sqrt{2\log T+n\log(1+\frac{t}{n\lambda})}$ for all $t\in[T]$, $H_1 = \lambda I$.

\For{$t=1,2,\dots,T$}{
     Pick 
     \begin{align*}
         a_t=
         \begin{cases}
            \argmin_{a\in\calA} \inner{a,\wh{\theta}_t} - \beta \|a\|_{H_t}^{-1}, &\mbox{in the loss case,} \\
            \argmax_{a\in\calA} \inner{a,\wh{\theta}_t} + \beta \|a\|_{H_t}^{-1}, &\mbox{in the reward case.}
        \end{cases}
     \end{align*}
     
     Observe the payoff $u_\tau$ for all $\tau$ such that $\tau+d_{\tau}\in (t-1,t]$.

     Update $H_{t+1} = H_t + \sum_{\tau:\tau+d_{\tau}\in(t-1,t]}a_{\tau}a_{\tau}^\top$ and $\wh{\theta}_{t+1}=H_{t+1}^{-1}\sum_{\tau:\tau+d_{\tau}\leq t}a_{\tau}u_{\tau}$.
     
}
\end{algorithm}

\end{document}